\documentclass[11pt]{article} 

\usepackage{fullpage}

\usepackage{amsmath,amsthm, amssymb, bbm}
\usepackage{latexsym}
\usepackage{graphicx}
\usepackage{ifthen}
\usepackage{subcaption}
\usepackage[numbers]{natbib}
\usepackage{mathtools}
\usepackage{dsfont}
\usepackage{nicefrac}
\usepackage{wrapfig}
\usepackage{mfirstuc}
\usepackage{hyperref}

\usepackage{algorithm} 
\usepackage{algorithmic}  
\floatname{algorithm}{Method}
\usepackage{soul}

\newcommand{\mainAlg}{{\textsc{Dual}}}

\DeclareMathOperator{\argmax}{argmax}
\DeclareMathOperator{\argmin}{argmin}

\newcommand{\OPT}{\texttt{OPT}}
\newcommand{\ubOPT}{\overline{\texttt{OPT}}}

\newcommand{\ug}{\underline g}
\newcommand{\ugone}{\ell}
\newcommand{\ugtwo}{\underline g}

\newcommand{\N}{\mathbb{N}}                     
\newcommand{\R}{\mathbb{R}}                     
 
\renewcommand{\S}{\mathcal{S}}

\newcommand{\T}{\mathcal{T}}

\newcommand{\bx}{\mathbf{x}}

\newtheorem*{theorem*}{Theorem}
\newtheorem*{definition*}{Definition}

\newtheorem{lemma}{Lemma}
\newtheorem{theorem}{Theorem}
\newtheorem{definition}{Definition}

\newtheorem{proposition}{Proposition}
\usepackage{xcolor}

\usepackage{thm-restate}

\newtheorem{innercustomthm}{Theorem}
\newenvironment{customthm}[1]
  {\renewcommand\theinnercustomthm{#1}\innercustomthm}
  {\endinnercustomthm}

\graphicspath{ {./images/} }

\begin{document}
	\title{Instance Specific Approximations for Submodular Maximization}
	\author{ Eric Balkanski \footnote{Columbia University, eb3224@columbia.edu} 
\and Sharon Qian \footnote{Harvard University, sharonqian@g.harvard.edu}
\and Yaron Singer \footnote{Harvard University, yaron@seas.harvard.edu}}
\date{}	
\setcounter{page}{0}
\maketitle	
\begin{abstract}
For many  optimization problems in machine learning, finding an optimal solution is computationally intractable and we seek algorithms that perform well in practice. Since computational intractability often results from pathological instances, we look for methods to benchmark the performance of algorithms against optimal solutions on real-world instances. The main challenge is that an optimal solution cannot be efficiently computed for intractable problems, and we therefore often do not know how far a solution is from being optimal. A major question is therefore how to measure the performance of an algorithm in comparison to an optimal solution on instances we encounter in practice.

In this paper, we address this question in the context of submodular optimization problems. For the canonical problem of submodular maximization under a cardinality constraint, it is  intractable to compute a solution that is better than a $1-1/e \approx 0.63$ fraction of the optimum. Algorithms like the celebrated greedy algorithm are guaranteed to achieve this $1-1/e$ bound on any instance and are used in practice.

Our main contribution is not a new algorithm for submodular maximization but an analytical method that measures how close an algorithm for submodular maximization is to optimal on a given problem instance. We use this method to show that on a wide variety of real-world datasets and objectives, the approximation of the solution found by greedy goes well beyond $1-1/e$ and is often at least $0.95$. We develop this method using a novel technique that lower bounds the objective of a dual minimization problem to obtain an upper bound on the value of an optimal solution to the primal maximization problem.

\end{abstract}
\newpage


\section{Introduction}

A central challenge in machine learning is that many of the optimization problems we deal with are  computationally intractable. For problems like clustering, sparse recovery, and maximum likelihood estimation for example, finding an optimal solution is computationally intractable and we seek heuristics that perform well in practice. Computational intractability implies that under worst case analysis any efficient algorithm is suboptimal; however, worst-case approximation guarantees are often due to pathological instances that are not representative of instances we encounter in practice. Thus, we would like to be assured that the algorithms we use, despite poor performance on pathological instances, perform provably well on real-world instances.

In order to evaluate the performance of an algorithm on real-world instances, we would like to measure its performance in comparison to an optimal solution. The main challenge, however, is that we cannot evaluate the performance of an algorithm against an optimal solution since finding an optimal solution for a computationally intractable problem is, by definition, intractable. Thus, we often do not know how far an algorithm's solution is from optimal, and whether there is a substantially better algorithm. Therefore, for computationally intractable problems, our main challenge is not necessarily how to design better algorithms, but rather how to measure the performance of an algorithm in comparison to a theoretically optimal solution on real-world instances.

\begin{center}
	
	\emph{How do we measure the performance of an algorithm on specific instances \\for problems that are intractable?}
	
\end{center}

In this paper, we develop a method to measure how close to optimal the performance of an algorithm is on specific instances for the broad class of submodular maximization problems. In machine learning, many objectives that we aim to optimize, such as coverage, diversity, entropy, and graph cuts are submodular. As a result, submodular maximization algorithms are heavily employed in applications such as speech and document summarization \cite{lin2011class}, recommender systems \cite{mirzasoleiman2016fast},  feature selection \cite{das2011submodular}, sensor placement \cite{guestrin2005near}, and network analysis \cite{kempe2003maximizing}. 

Submodular maximization provides an ideal framework to address our main question 
because it is  intractable to compute a solution that is better than a $1-1/e$ approximation for the canonical problem of maximizing a monotone submodular function under a cardinality constraint \cite{nemhauser1978best}. In addition, multiple algorithms are known to enjoy constant factor approximation guarantees, such as the greedy algorithm that achieves this $1-1/e$ approximation on any instance \cite{NWF78}. Even though greedy is widely used in practice, we do not know how close to optimal its performance  is on the instances we encounter, except that it finds a solution of value that is at least a $1-1/e$ fraction of the optimal value.

\paragraph{Our contribution.} We develop a novel and efficient method, called \textsc{Dual}, to measure how close to optimal the performance of an algorithm is on an instance of maximizing a monotone submodular function under a cardinality constraint. This instance specific approximation is obtained by upper bounding the optimal  value of an instance.
We use this method to show that greedy, as well as other submodular maximization algorithms, perform significantly better than $1-1/e$ in practice. On a wide variety of large real-world datasets and objectives, we find that the approximation of the solution found by greedy almost always exceeds $0.85$ and often exceeds $0.95$, a $50$ percent improvement over $1-1/e \approx 0.63$. 
Additionally, we show that \mainAlg \ significantly outperforms multiple benchmarks for measuring instance specific approximations.

\subsection{Technical overview} Given an instance of the optimization problem $\max_{|S| \leq k} f(S)$, where $f: 2^N \rightarrow \R$ is a monotone submodular function,  \mainAlg \ measures how close to optimal a solution $S$ is by upper bounding the optimal value $\OPT$ of the problem. We take a primal-dual approach to upper bounding $\OPT$ that lower bounds the optimal value of a dual minimization problem. The dual problem that we consider is $g(v) = \min_{S: f(S) \geq v} |S|$, which consists of finding the solution $S$ of minimum size that has value at least $v$. The main technical part of our approach is the  construction of a function $\ug(v)$ that  lower bounds $g(v)$  and  is efficiently computable for all values $v$. Given such a function $\ug(v)$, we then find the maximum value  $v^{\star}$ such that $\ug(v^{\star}) \leq k$, which is the upper bound on $\OPT$ used to measure how close a solution $S$ is to optimal. 



In Section~\ref{sec:coverage}, we first consider coverage functions,  which are a subclass of submodular functions where the goal is to maximize the coverage of a universe $U$. For coverage functions, we consider the dual objective $g : 2^U \rightarrow \R$ that consists of finding the minimum size of  a set $S \subseteq N$  that covers $T \subseteq U$. This dual objective is a special case of $g(v) = \min_{S: f(S) \geq v} |S|$ that has additional structure since it is defined over a dual space of elements $U$. We take advantage of this additional structure to construct lower bounds on $g$. Our first lower bound on $g$ is an \emph{additive} function $\ell : 2^U \rightarrow \R$ over the dual space of elements of $U$. Since $\ell$ is additive, it can be minimized efficiently to give a lower bound on the dual problem.  We then improve $\ell(T)$ by construction a tighter, more sophisticated, lower bound $\ug(T)$ that can still be minimized efficiently. This lower bound is based on partitioning the dual space $T \subseteq U$ into parts  $P_i$, $\cup_{i=1}^kP_i = T$. These parts are such that, for all $i \in [k]$, there is no  element $a \in N$ that can cover more than $|P_i|$  elements in $U \setminus \cup_{j=1}^iP_j$.

In Section~\ref{sec:submodular}, we generalize the lower bound $\ug(T)$ on the dual objective $g : 2^U \rightarrow \R$ for coverage functions to a lower bound $\ug(v)$ on the dual objective $g(v)$ for general submodular functions. Instead of partitioning universe $U$, which is specific to coverage functions, the lower bound $\ug(v)$ for submodular functions partitions value $v$ into $k$ values $v_1, \ldots, v_k$, $\sum_i v_i = v$. These values are such that, for each $i \in [k]$, there is no set $S \subseteq N$ of size $i$ that satisfies $f(S) > \sum_{j=1}^i v_i$.

\subsection{Related work} 

\paragraph{Explanations for the performance of greedy in practice.} A closely related line of work has investigated different properties of submodular functions that enable improved approximation guarantees for the greedy algorithm. The curvature $c \in [0,1]$ of a function $f$ measures how close $f$ is to additive \cite{CC84}. Submodular sharpness, an analog of sharpness from continuous optimization \cite{L63},  measures the behavior of a function $f$ around the set of optimal solutions \cite{PST19}. Finally,  a function $f$ is perturbation-stable if the optimal solution for maximizing $f$ does not change under small perturbations \cite{CRV17}. These parameterized properties all yield improved approximation guarantees for the greedy algorithm, with the additional benefit that they provide an explanation for the improved performance. However, the main issue with using these properties to measure greedy performance on specific instances is that the parameters of these properties cannot be computed efficiently. Since they require brute-force computation, these parameters have only been computed on small instances with at most $n = 20$ elements \cite{PST19} and  cannot be computed on real-world instances. In addition, on these small  instances, they yield approximations that are not as strong as those obtained by \mainAlg.

\paragraph{Continuous extensions.} For problems such as max-coverage and  traveling salesman problem  that can be formulated as integer linear programs,  we can use the LP relaxation of these formulations to obtain a bound on the optimal solution and use this bound to measure how close a solution is to optimal. Submodular functions have multiple continuous extensions but, unlike the LP relaxation of integer programs, these continuous extensions cannot be maximized efficiently. For the concave closure $F^+ : [0,1]^n \rightarrow \R$ of a submodular function $f$, it is  APX-hard to even evaluate $F^+(\bx)$ \cite{calinescu2007maximizing,vondrak2007submodularity}. The multilinear extension $F: [0,1]^n \rightarrow \R$ can be estimated arbitrarily well and is widely used in submodular maximization (e.g. \cite{vondrak2008optimal}), but  $\max_{\bx \in[0,1]^n:\|\bx\|_1 \leq k} F(\bx)$ cannot be approximated better than $1-1/e$. 

\paragraph{Practical submodular maximization.} Primarily motivated by  applications in machine learning, there have recently been multiple lines of work  on making submodular maximization algorithms more practical. The running time of submodular maximization has been reduced by improving the number of function evaluations \cite{minoux1978accelerated, MBKV15, buchbinder2015comparing}. Lines of work on distributed \cite{kumar2015fast, mirzasoleiman2013distributed, mirrokni2015randomized, barbosa2016new, liu2018submodular}, streaming \cite{badanidiyuru2014streaming, chekuri2015streaming, MBNTC17, FKK18, kazemi2019submodular}, and parallel \cite{BS18,  chekuri2019,  chen2018, ene2018, farbach2019, balkanski2019exponential} algorithms for submodular maximization address multiple challenges associated with large scale optimization. Motivated by applications where the objective is learned from data, recent lines of work have studied submodular optimization under noise \cite{hassidim2017submodular,horel2016maximization,hassidim2018optimization} and from samples \cite{balkanski2017limitations, balkanski2016power}. Different models for robust submodular optimization have also been considered \cite{bogunovic2017robust, mirzasoleiman2017deletion, chen2017robust}. Finally, maximizing weakly submodular objectives has been studied and captures problems such as feature selection \cite{das2011, elenberg2018, qian2019fast}.

\subsection{Preliminaries}

A function $f: 2^N \rightarrow \R$ is submodular if $f_S(a) \geq f_T(a)$ for all $S \subseteq T \subseteq N$ and $a \in N \setminus T$, where $f_S(a) = f(S \cup \{a\}) - f(S)$ is the marginal contribution of $a$ to $S$. It is monotone if $f(S) \leq f(T)$ for all $S \subseteq T \subseteq N$.  A function $f: 2^P \rightarrow \R$ is a coverage function if there  exists a bipartite graph $G = (P, D, E)$ over primal and dual elements $P \cup D$ such that $f(S) = |N_G(S)|$ where $N_G(S) = \cup_{a \in S} N_G(a) \subseteq D$ denotes the neighbors of $S \subseteq P$ in $G$. We say that set $S$ covers $T$, or equivalently that $T$ is covered by $S$, if  $T \subseteq N_G(S)$. A function $f: 2^N \rightarrow \R$ is additive if $f(S) = \sum_{a \in S} f(a)$.

Given  an instance of the problem $\max_{|S| \leq k} f(S)$ and a solution $S^{\star}$ to this problem, we aim to compute an approximation for $S^{\star}$, i.e. a lower bound on $f(S^{\star})/\max_{|S| \leq k} f(S)$ or, equivalently, an upper bound on  $\max_{|S| \leq k} f(S)$.

\section{Instance Specific Approximations for Max-Coverage}
\label{sec:coverage}

In this section, we develop a method that measures how close to optimal a solution  to an instance of maximum coverage is. The special case of coverage functions motivates and provides intuition for the main ideas behind the method for submodular functions in Section~\ref{sec:submodular}.  In Section~\ref{sec:dualcoverage}, we introduce the problem of minimum cover under a cardinality constraint, which is a generalization of the classical set cover problem. We show that a lower bound on this minimum cover problem implies an upper bound on the optimal value $\OPT$ for the max-coverage problem. In Section~\ref{sec:linearcoverage}, we present an additive lower bound on the dual problem, which is used to efficiently compute a lower bound of the optimal value to the dual problem. Then, in Section~\ref{sec:partitioncoverage}, we develop a more sophisticated lower bound on the dual problem.

\subsection{The dual problem}
\label{sec:dualcoverage}

We introduce the minimum cover under a cardinality  constraint problem. Recall that given a bipartite graph $G$ over  nodes $P$ and nodes $D$, the problem of maximum coverage  under a cardinality constraint  problem is to find the $k$  elements $S \subseteq P$ that maximize the number of elements $N_G(S) \subseteq D$ covered by $S$. In contrast, the minimum cover under a cardinality  constraint  problem is to find the $v$ elements $T \subseteq D$ that minimize the number of  elements $S \subseteq P$ needed to cover $T \subseteq N_G(S)$. When $G$ is clear from the context, we write $N(S)$ instead of $N_G(S)$ to denote the neighbors of $S$ in graph $G$.
\begin{definition}
	\label{def:mincover}
	The minimum cover under a cardinality constraint problem is defined as $$\min_{T \subseteq D: |T| \geq v} g(T),$$ where $$g(T) = \min_{S \subseteq P : T \subseteq N(S)} |S|$$ is the size of the minimum cover of $T$ and where $v \in [|D|]$ is the cardinality constraint.
\end{definition}
This  problem is a generalization of the classical set cover problem, which finds the minimum number of  elements $S$ to cover \emph{all} elements $D$.   We obtain the following duality property: a lower bound on  minimum cover  implies an upper bound on maximum coverage, and vice-versa.

\begin{lemma}\label{lem:duality}
	Let $f: 2^P \rightarrow \N$ be a coverage function defined over a biparite graph between elements $P$ and $D$. For any $k \in [|P|]$ and $v \in [|D|]$, $$\max_{S \subseteq P: |S| \leq k} f(S) < v \text{ if and only if }\min_{T \subseteq D: |T| \geq v}g(T) > k$$
	where $g: 2^D \rightarrow \N$ is the size of the minimum cover of $T$ as defined in Definition~\ref{def:mincover}. 
\end{lemma}
\begin{proof} We first prove that if $\min_{T \subseteq D: |T| \geq v}g(T) > k$ then $\max_{S \subseteq P: |S| \leq k} f(S) < v$. By contrapositive, 
	assume that $\max_{S \subseteq P: |S| \leq k} f(S) \geq v$. This implies that there exists $S^{\star}$ such that $f(S^{\star}) = |N(S^{\star})| \geq v$ and $|S^{\star}| \leq k$. We get 
	$$\min_{T \subseteq D: |T| \geq v} g(T) \leq g(N(S^{\star})) \leq |S^{\star}| \leq k.$$
	For the other direction, we again prove by contrapositive. Assume that $\min_{T \subseteq D: |T| \geq v}g(T) \leq k$. This implies that there exists $T^{\star}$ and $S^{\star}$ such that $|T^{\star}| \geq v$, $|S^{\star}| \leq k$, and  $T^{\star} \subseteq N(S^{\star})$. We get 
	\begin{align*}
		\max_{S \subseteq P: |S| \leq k} f(S) &\geq f(S^{\star}) \geq |T^{\star}| \geq v. \hfill \qedhere
	\end{align*}
\end{proof}
We refer to $\max_{S \subseteq P: |S| \leq k} f(S)$ and $\min_{T \subseteq D: |T| \geq v}g(T)$ as the primal and  dual problems. We also refer to $P$ and $D$ as the primal  and dual elements. 

\subsection{Warm-up: Approximations via an additive lower bound on the dual}\label{sec:linearcoverage}

This dual problem admits an additive lower bound that is, as we will show empirically in Section~\ref{sec:exp}, close to the dual objective in practice. This is in contrast to the primal maximum coverage problem, which is far from  additive on real instances.  We define the individual value $v_b$ of each dual element $b$ as 
$v_b = \min_{a \in N(b)} \frac{1}{|N(a)|}$. The additive function $\ugone: 2^D \rightarrow \R$ is defined as follows:
$$\ugone(T) = \sum_{b \in T} v_b = \sum_{b \in T} \min_{a \in N(b)} \frac{1}{|N(a)|}.$$
%
Note that if $b \in D$ is covered by primal element $a \in P$, then $a$  covers at most $1/v_b$ dual elements. In other words, $1/v_b$ is an upper bound on the value obtained from an element $a$ that covers $b$. 

We use this additive lower bound $\ugone(\cdot)$ on the dual objective to design a method that returns an upper bound on the optimal value for the primal problem. Method~\ref{alg:coverage} first orders the dual elements $b$ by increasing value $v_b$. It then finds the prefix $\{b_1, \ldots, b_{i^\star}\}$ of this ordering where $i^{\star}$ is the minimum size $i$ such that $\ugone(\{b_1, \ldots, b_i\}) = \sum_{j=1}^i v_{b_j}  > k$, and then returns  $i^{\star}$. In other words, it finds the largest size $i$ such that $\ugone(T) > k$ for all sets $T$ of size $i$.

\begin{algorithm}[H]
	\caption{Linear bound on dual objective for coverage}
	\label{alg:coverage}
	\begin{algorithmic}
		\INPUT  bipartite graph $G = (P, D, E)$,   constraint $k$
		\STATE  $v_{b} \leftarrow  \min_{a \in N(b)} \frac{1}{|N(a)|}$, for each $b \in D$
		\STATE  $(b_1, \ldots, b_{|D|}) \leftarrow$ elements $D$ ordered by increasing $v_{b_i}$
		\STATE $i^{\star} \leftarrow \min\{i: \sum_{j=1}^i v_{b_j} > k\}$
		\STATE \textbf{return} $i^{\star}$
	\end{algorithmic}
\end{algorithm}

\paragraph{The analysis.} We first show that $\ugone(\cdot)$ is a lower bound on the dual objective $g(\cdot)$ (Lemma~\ref{lem:cover_lowerbound}). We then show that $\{b_1, \ldots, b_i\}$ minimizes $\ugone$ over all sets of size at least $i$ (Lemma~\ref{lem:cov_optimal}). Together, these imply that there are no sets of primal elements of size $k$ which cover $i^{\star}$ dual elements and we obtain $i^{\star} > \max_{S \subseteq P: |S| \leq k} f(S)$ (Theorem~\ref{thm:coverage1}).

\begin{lemma}
	\label{lem:cover_lowerbound}
	For any coverage function defined by a bipartite graph $G$ between $P$ and $D$, we have that for  any set $T \subseteq D$, $\ugone(T) \leq g(T) = \min_{ S \subseteq P : T \subseteq N(S)} |S|.$
\end{lemma}
\begin{proof} For any $S \subseteq P$ such that $T \subseteq N(S)$, we have
	\begin{align*}
		|S| & = \sum_{a \in S} \sum_{b \in N(a)} \frac{1}{|N(a)|}  \\
		& \geq \sum_{a \in S} \sum_{b \in N(a)} \min_{a' \in N(b)} \frac{1}{|N(a')|} \\
		& \geq \sum_{b \in T} \min_{a \in N(b)} \frac{1}{|N(a)|} \\
		&= \ugone(T). \qedhere
	\end{align*}
\end{proof}
 
\begin{lemma}
	\label{lem:cov_optimal}
	Consider the ordering of dual elements $D$ by increasing singleton values, i.e., $(b_1, \ldots, b_{|D|})$ where $\ugone(b_i) \leq \ugone(b_j)$ for all $i < j$, then, for all $v \leq |D|$,
	$\ugone(D_v) = \min_{T \subseteq D: |T| \geq v} \ugone(T)$ 
	where $D_v = \{b_1, \ldots, b_v\}$.
\end{lemma}
\begin{proof}
	Since $\ugone$ is an additive function, the set $T$ of size at least $v$ with minimum value is the set consisting of the $v$ dual elements of minimum singleton value, which is $D_v$.
\end{proof}



We are now ready to formally prove that Method~\ref{alg:coverage} returns an upper bound on the optimal value to the primal problem.

\begin{theorem}
	\label{thm:coverage1}
	For any $k$, let $i^{\star}$ be the value returned by Method~\ref{alg:coverage}, then  $i^{\star} > \max_{S \subseteq P: |S| \leq k} f(S)$. 
\end{theorem}
\begin{proof}
	By definition of $i^{\star}$, Lemma~\ref{lem:cov_optimal} and Lemma~\ref{lem:cover_lowerbound}, 
	\begin{align*}
		k < \ugone(D_{i^{\star}}) & = \min_{\substack{T \subseteq D: |T| \geq i^{\star}}} \ugone(T) \leq  \min_{\substack{T \subseteq D: |T| \geq i^{\star}}} \min_{\substack{ S \subseteq P : \\T \subseteq N(S)}} |S|.
	\end{align*}
	By Lemma~\ref{lem:duality}, we get $i^{\star} > \max_{S \subseteq P: |S| \leq k} f(S)$. 
\end{proof}

\subsection{Improved method for coverage functions} \label{sec:partitioncoverage}

We improve  Method~\ref{alg:coverage} by constructing a lower bound $\ugtwo(\cdot)$ on the dual objective $g(\cdot)$ that is tighter than $\ell(\cdot)$. The function $\ugtwo(T)$ is obtained by partitioning the collection of dual elements $T$ into $j$ parts $\cup_{i=1}^j P_i = T$. We define the weight $w(P_i)$ of part $P_i$  to be 
$$w(P_i) = |P_i| \cdot \max_{b \in P_i} v_b = |P_i| \cdot \max_{b \in {P_i}} \min_{a \in N(b)} \frac{1}{|N(a)|}.$$

We note that if $w(P_i) > 1$, then dual elements $P_i$ cannot be covered by a single primal element since there must exist $b \in P_i$ such that $\max_{a \in N(b)} |N(a)| < |P_i|$. This motivates the following definition of a valid partition.

\begin{definition} 
	A partition $P_1, \ldots, P_j$ of  $T$ is \emph{valid} if $w(P_i) \leq 1$  for all  $i \leq j$.
\end{definition}

This definition is such that if a partition $P_1, \ldots, P_j$ is \emph{not} valid, then there must exist a part $P_i$ which cannot be covered  by a single primal element. We exploit this property to define the following improved lower bound $\ugtwo: 2^D \rightarrow \R$ on the dual objective:
$$\ugtwo(T) =\min\{j : \exists  \text{ a valid partition } P_1, \ldots, P_j \text{ of } T\}.$$
This lower bound $\ugtwo(T)$ is always tighter than the additive lower bound $\ugone (\cdot)$.

\begin{proposition}
	For any coverage function $f: 2^P \rightarrow \N$ defined by bipartite graph $(P, D, E)$, we have $\ugtwo(T) \geq \ugone(T)$ for all $T \subseteq D$.
\end{proposition}

Similarly as with Method~\ref{alg:coverage}, we use this lower bound $\ugtwo(\cdot)$ on the dual objective to design a method that returns an upper bound on the optimal value for the primal problem.  Method~\ref{alg:coverage2}, which will be generalized to Method~\ref{alg:sm} for submodular functions, iteratively constructs a valid partition $\cup_{\kappa=1}^k P_\kappa$ of a collection of dual elements $T$  such that $|T|$ is maximized. At iteration $\kappa$, part $P_\kappa$ of the partition is defined as the  collection of dual elements $\{b_{i_{\kappa-1} + 1}, \ldots, b_{i_{\kappa}} \}$, which are the dual elements with minimum value $v_b$ that are not in the previous parts $P_1, \ldots, P_{i-1}$, where $i_{\kappa}$ is the maximum index such that part $P_i$ is valid. The method returns value $i_k$, which is the total size $|\cup_{\kappa=1}^kP_\kappa|$ of the partition.

\begin{algorithm}[H]
	\caption{Dual bound via partitioning for coverage}
	\label{alg:coverage2}
	\begin{algorithmic}
		\INPUT  bipartite graph $G = (P, D, E)$,   constraint $k$
		\STATE    $v_{b} \leftarrow  \min_{a \in N(b)} \frac{1}{|N(a)|}$, for each $b \in D$
		\STATE  $(b_1, \ldots, b_{|D|}) \leftarrow$ elements $D$ ordered by increasing $v_{b_i}$
		\STATE $i_0 \leftarrow 0$
		\STATE \textbf{for} $\kappa = 1$ to $k$ \textbf{do}
		\STATE \qquad$i_\kappa \leftarrow \max\{i: (i-i_{\kappa-1} ) \cdot v_{b_i} \leq 1\}$
		\STATE \qquad $P_\kappa \leftarrow \{b_{i_{\kappa-1} + 1}, \ldots, b_{i_{\kappa}} \}$
		\STATE \textbf{return}  $i_k$
	\end{algorithmic}
\end{algorithm}

Since Method~\ref{alg:coverage2}  is a special case of Method~\ref{alg:sm}, the analysis of Method~\ref{alg:sm} in the next section also applies to Method~\ref{alg:coverage2}.

\section{Instance Specific Approximations for Submodular \\ Maximization}\label{sec:submodular}

In this section, we present our main method, \mainAlg, which generalizes Method~\ref{alg:coverage2} to submodular functions. For coverage functions, the value achieved by a solution  $S$ corresponds to the number of dual elements covered by $S$ and the optimal value can be upper bounded by analyzing dual elements. However, for general submodular functions, there are no dual elements corresponding to the solution value.   We first introduce a similar dual minimization problem as for coverage, but defined over values $v \in \R$ instead of dual elements $T \subseteq D$. We then construct a lower bound  on the dual objective and use it to design a method that upper bounds $\OPT$ in $O(n \log n)$ running time.

We introduce the minimum submodular cover under a cardinality constraint problem, where the goal is to find the smallest set $S$ of value at least $v$:
$$g(v) = \min_{S\subseteq N: f(S) \geq v} |S|.$$
This problem is a generalization of the submodular cover problem \cite{wolsey1982analysis}, which is to find the smallest set $S$ of value at least $f(N)$.
We note that, unlike the dual objective for coverage functions, there are no dual elements. 

 Next, we define a function $\ug(v)$ which lower bounds this dual objective. Similarly as for coverage functions, we consider partitions of the dual space  and we define a collection of valid partitions  that is used to then define $\ug(v)$. We assume that the ground set of elements $N = \{a_1, \ldots, a_n\}$ is indexed by decreasing singleton value, i.e., $f(a_i) \geq f(a_j)$ for all $i < j$. We define $A_i := \{a_1, \ldots, a_i\}$.

\begin{definition}
	Values $v_1,  \ldots, v_k \in \R$ form a valid partitioning of  $v \in \R$ if $\sum_{j \in [k]} v_j = v$ and there exists a witness $W \subseteq N$ such that, for all $j \in [k]$, $f(a_{i_j}) \geq v_j$ where 
	$$i_j = \min\left\{i: a_i \in W, f(W \cap A_i) \geq \sum_{\ell = 1}^j v_\ell\right\}.$$
\end{definition}
As we will show in Lemma~\ref{lem:lb}, for any solution set $S$, $v_1 = f_{S_{0}}(S_1), v_2 = f_{S_{1}}(S_2), \ldots,  v_{|S|} = f_{S_{|S|-1}}(S)$ forms a valid partioning of $v = f(S)$, where $S_j$ is the set of $j$ elements in $S$ with the largest singleton value. This implies that if a partition $v_1, \ldots, v_k$ is not valid, then there is no solution $S$ of size $k$ such that $f_{S_{j-1}}(S_j) \geq v_j$ for all $j \in [k]$. Thus, if a value $v$  does not have a valid partitioning $v_1, \ldots, v_k$, then there are no solution $S$ of size $k$ such that $f(S) \geq v$. This implies that
the following function $\underline g : \R \rightarrow \N$ lower bounds the dual objective (Lemma~\ref{lem:lb}):
$$\underline g(v) = \min\{k: \exists \text{ a valid partition } v_1,  \ldots, v_k \text{ of } v\}.$$
We use the lower bound $\ug(v)$ on the dual objective to construct a method that returns an upper bound on $\OPT$. Method~\ref{alg:sm} iteratively constructs a valid partition $v_1,  \ldots, v_k$ of the dual space such that $\sum_{j \in [k]} v_j$ is maximized. It first orders the elements $a_i \in N$ by decreasing singleton value $f(a_i)$. Then, at each iteration $j$, it defines value $v_j$ to be the maximum value $v$ such that  partition $v_1,  \ldots, v_j$ is a valid partition with witness $W = A_{i_j}$.

\begin{algorithm}[H]
	\caption{Method for submodular functions}
	\label{alg:sm}
	\begin{algorithmic}
		\INPUT   function $f$,  cardinality constraint $k$
		\STATE  $(a_1, \ldots, a_n) \leftarrow$ elements ordered by decreasing $f(a_i)$
		\STATE \textbf{For} $j = 1$ to $k$ \textbf{do}
		\STATE \qquad $v_j \leftarrow \max\{v: f(a_{i_j}) \geq v \text{ where }$
		\STATE \qquad \qquad \qquad \quad  $i_j  = \min \{i: f(A_i) - \sum_{\ell = 1}^{j-1} v_\ell \geq  v \}\}$
		\STATE \textbf{return}  $\sum_{j=1}^k v_j$
	\end{algorithmic}
\end{algorithm}

Value $v_j$ at iteration $j$ can be found by iterating through  elements indexed by $i \in \{i_{j-1}+1, i_{j-1}+2, \ldots \}$ until $i^\star$, where $i^{\star}$ is the minimum index such that $f(A_{i^{\star}}) - \sum_{\ell = 1}^{j-1} v_\ell \geq f(a_{i^{\star}})$. If $f(A_{i^{\star}-1}) - \sum_{\ell = 1}^{j-1} v_\ell < f(a_{i^{\star}})$, we then have  $v_j = f(a_{i^{\star}})$, otherwise we decrement $i^\star$ by one and let $v_j  = f(A_{i^{\star}}) - \sum_{\ell = 1}^{j-1} v_\ell$. 
Since an element $a_i$ is considered at most once over all iterations, the total running time of the for loop  is $O(n)$. Thus, the running time of Method~\ref{alg:sm} is $O(n \log n)$ due to the sorting of the elements by singleton value.
More details on finding value $v_j$ in Appendix \ref{app:pseudo}.

\paragraph{The analysis.} 
We first show that $\ug(v)$ lower bounds $g(v)$ in Lemma~\ref{lem:lb}. The proof uses the fact that $\ug(v)$ is monotone (Lemma~\ref{lem:monotone}). 
\begin{lemma}
	\label{lem:monotone}
	$\ug(v)$ is monotonically increasing.
\end{lemma}
\begin{proof}
We prove this by contradiction. Let $v \geq u$. We assume that $\ug(v) < \ug(u)$, i.e. the minimal valid partition of $v$ is smaller than the minimal valid partition of $u$.

By definition of $\ug$, there exists a valid partition $v_1,.., v_{\ug(v)}$ of $v$ with witness $S$. Since $u < v$, for some $i < \ug(v)$, $\sum_{j=1}^i v_j < u < \sum_{j=1}^{i+1} v_j$. By the definition of partition of $v$, for all $j \in [i]$, $S$ is a witness of partition $v_1,..,v_i$. Additionally, $v' = u - \sum_{j=1}^i v_i <= v_{i+1}$ and we have that $v_1,...v_i, v'$ is a valid partition of $u = \sum_{j=1}^i v_j + v'$ with partition size $k$, where $k \leq \ug(v) < \ug(u)$. This contradicts the fact that $\ug(u)$ minimal valid partition size for value $u$. Thus, we have $\ug(u) \leq \ug(v)$ for $v \geq u$.
\end{proof}

\begin{lemma}\label{lem:lb}
	For any submodular function $f$ and  $v \in \R$, $\ug(v) \leq g(v) = \min_{S \subseteq N: f(S) \geq v} |S|.$
\end{lemma}
\begin{proof}
	Consider any set $S = \{a'_{1}, \ldots, a'_{|S|}\}$ indexed so that $f(a'_{i}) \geq f(a'_{j})$ for all $i < j$. We consider the partition $v_1,  \ldots, v_{|S|}$ of $f(S)$ defined by $v_j = f_{S \cap A_{i_{j-1}}}(a_{i_j}) = f_{\{a_{i_1}, \ldots a_{i_{j-1}}\}}(a_{i_j})$.
	
	Note that by definition of $S$, $f_{S \setminus a'}(a') > 0$ for all $a' \in S$. By submodularity, this implies  $f_{S \cap A_{i-1}}(a'_i) > 0$ for all $a'_i \in S$. Thus, 
	$f(S \cap A_i) < f(S \cap A_{i_j}) = \sum_{\ell  = 1}^j v_\ell$ for all $i < i_j$. This implies  $i_j = \min\{i: a'_i \in S, f(S \cap A_i) \geq \sum_{\ell = 1}^j v_\ell\}.$ By submodularity, $f(a_{i_j}) \geq f_{S \cap A_{i_{j-1}}}(a_{i_j}) = v_j$. Thus, $v_1,  \ldots, v_{|S|}$, with witness $W=S$, is a valid partition of $f(S)$. 
	
	Using Lemma \ref{lem:monotone} and considering the case where $S^{\star} = \argmin_{S \subseteq N: f(S) \geq v} |S|$, we get
	\begin{align*}
		\ug(v) \leq \ug(f(S^{\star})) \leq |S^{\star}| &= \min_{S \subseteq N: f(S) \geq v} |S|. \qedhere
	\end{align*}
\end{proof}
%

We now show that the value $\sum_{j=1}^k v_j$ returned by the method is the maximum value $v$ such that $\ug(v) \leq k$ (Lemma \ref{lem:ub}). The proof of Lemma~\ref{lem:ub} utilizes Lemma~\ref{lem:witness}, which shows that for any valid partition, there is an index $i$ such that set $A_i$ is a witness of that partition.
\begin{lemma}
	\label{lem:witness}
	For any $v \in \R$ and for any valid partition $v_1,  \ldots, v_k$ of $v$, there exists $i$ such that $A_i$ is a witness of  partition $v_1,  \ldots, v_{k}$.
\end{lemma}
\begin{proof}
	Consider a witness $W$ of a valid partition $v_1,  \ldots, v_{k}$ of $v$ such that $a_i \not \in W$ and $a_j \in W$ for $i < j$. We claim that $W \cup a_i$ is also a witness of partition $v_1,  \ldots, v_{k}$. Let $i_1, \ldots, i_{k}$ and $i'_1, \ldots, i'_{k}$ be the indices of $W$ and $W \cup a_i$ respectively for a valid partition of $v$. By monotonicity, we have $i_j \geq i'_j$ for all $j \in [k]$. By the ordering of the indices of the elements by decreasing singleton values, we get
	$f(a_{i'_j}) \geq f(a_{i_j}) \geq v_j$. Thus, $W \cup a_i$ is indeed a witness of partition $v_1,  \ldots,  v_{k}$. 
\end{proof}
	
\begin{lemma}\label{lem:ub}
	Let $\ubOPT = \sum_{j=1}^k v_j$ be the solution returned by Method~\ref{alg:sm}, then 
	$\ug(\ubOPT +\epsilon) > k$ for all $\epsilon > 0$.
\end{lemma}
\begin{proof}
	 Assume by contradiction that there exists $\epsilon > 0$ such that $\ug(\ubOPT + \epsilon) \leq k$. Then there exists a valid partition $v_1' ,  \ldots,  v'_k$ of $\ubOPT + \epsilon = \sum_j v'_j$ with witness $A_{i}$ for some $i$ by Lemma 7. Let $j^{\star}$ be the minimum index $j$ such that $v'_j > v_j$. Then contradiction with the method and the definition of valid partitioning.
	\end{proof}
Finally, we combine Lemma~\ref{lem:lb} and Lemma~\ref{lem:ub} to show that  Method \ref{alg:sm} returns an upper bound of \texttt{OPT}.
\begin{theorem}
	\label{thm:sm1}
	For any $k$, Let $\ubOPT = \sum_{j=1}^k v_j$ be the solution returned by Method~\ref{alg:sm}, then  $\ubOPT \geq \max_{S \subseteq N: |S| \leq k} f(S)$.
\end{theorem}
\begin{proof}
We first show by contrapositive that a lower bound $k < g(v) = \min_{S \subseteq N: f(S) \geq v} |S|$ on the dual problem implies and upper bound $v > \max_{S \subseteq N: |S| \leq k} f(S)$ on the primal problem. 
Assume $\max_{S \subseteq N: |S| \leq k} f(S) \geq v$. Then, there exists $S^\star$ such that $f(S^\star) \geq v$ and $|S^\star| \leq k$ and we get $g(v) = \min_{S \subseteq N: f(S) \geq v} |S| \leq |S^\star| \leq k$. 

Then, by Lemma \ref{lem:ub} and Lemma \ref{lem:lb}, we have 
$$k <  \ug(\ubOPT +\epsilon) \leq g(\ubOPT +\epsilon) = \min_{S \subseteq N: f(S) \geq \ubOPT +\epsilon} |S|$$
which implies $\ubOPT + \epsilon > \max_{|S| \leq k} f(S)$
for all $\epsilon > 0$.
\end{proof}

\subsection{\mainAlg}

We describe our main method, \mainAlg, which uses Method~\ref{alg:sm} as a subroutine. In the case where a small number of elements have very large singleton values, Method~\ref{alg:sm} can return an arbitrarily bad approximation to $\OPT$ (See example in Appendix \ref{app:bad}). To circumvent this issue, \mainAlg \ calls Method~\ref{alg:sm} on the marginal contribution function $f_S(T) = f(S \cup T) - f(S)$ for each $S$ in a collection of sets $\mathcal S$ given as input. If $\{\} \in \S$, then \mainAlg \ is no worse than Method~\ref{alg:sm}. If there is $S \in \S$ such that there are no elements with large singleton value according to $f_S$, then \mainAlg \ circumvents the issue previously mentioned. We  note that adding more sets to $\mathcal S$ can only improve the approximation given by \textsc{Dual}.

\begin{algorithm}[H]
	\caption{\textsc{Dual}}
	\label{alg:3}
	\begin{algorithmic}
		\INPUT   function $f$, constraint $k$, collection of sets $\mathcal S$
		\STATE $\ubOPT \leftarrow f(N)$
		\STATE \textbf{for}  $S$ in $\mathcal S$ \textbf{do}
		\STATE \qquad $\ubOPT' \leftarrow \text{Method}\ref{alg:sm}(f_S, k)$
		\STATE \qquad $\ubOPT \leftarrow \min(f(S) + \ubOPT', \ubOPT)$
		\STATE \textbf{return}  $\ubOPT $
	\end{algorithmic}
\end{algorithm}
\begin{theorem}\label{thm:dual}
	For any set collection of sets $ \mathcal S$,  Method~\ref{alg:3} returns $\ubOPT$ such that $\ubOPT \geq \OPT$.
\end{theorem}
\begin{proof}
	Since $f$ is monotone submodular, $f_S$ is also monotone submodular for any set $S$. Thus the value $\ubOPT'$ returned by $\text{Method}\ref{alg:sm} (f_S, k)$ is such that $\ubOPT'\geq \max_{|T| \leq k} f_S(T)$ by Theorem \ref{thm:sm1}. By monotonicity we have that $f(S) + f_S(T) = f(T \cup S) \geq f(T)$ for all $T$. We conclude that at each iteration, $f(S) + \ubOPT' \geq \OPT$.
\end{proof}
We also show a guarantee on how far the upper bound $\ubOPT$ given by \textsc{Dual} is to $\OPT$. If $S_g \in \S$, where $S_g$ is the  greedy solution, then $\frac{1}{2} \ubOPT \leq \OPT$. In addition, we also get the stronger guarantee that \mainAlg \ always finds an instance specific approximation  for  greedy that is at least $1/2$.
\begin{proposition}\label{prop:dual}
	Let $S_g$ be the solution retuned by the greedy algorithm to the problem $\max_{|S| \leq k} f(S)$. Then, if $f$ is a monotone submodular function and $S_g \in \S$, \mainAlg \ returns $\ubOPT$ such that  $$\frac{1}{2} \ubOPT \leq f(S_g) \leq \OPT.$$
\end{proposition}
\begin{proof}
Let $S_i$ be the \textsc{Greedy} solution set of size $i$ and $a_i$ be the element chosen in the $i$-th iteration and $S_g$ be the solution at iteration $k$.

In the $k$-th iteration of \textsc{Dual}, $\ubOPT \leq f(S_g) + \ubOPT'$, where $\ubOPT'$ is the output of Method \ref{alg:sm} on the marginal contribution function $f_{S_g}$. Since $S_g$ is the \textsc{Greedy} solution, for all $a \in N$, $f_{S_g}(a) \leq f_{S_{k-1}}(a) \leq f_{S_{k-1}}(a_k)$, where $a_k$ has the largest marginal contribution at iteration $k$.

Then we can see that $$\ubOPT' = \sum_{j=1}^k v_j \leq kf_{S_{k-1}}(a_k) \leq \sum_{i=1}^k f_{S_{i-1}} (a_i) = f(S
_g)$$ and that $\ubOPT \leq 2 f(S_g) \leq 2 \texttt{OPT}$, yielding the result desired.
\end{proof}


\section{Experiments}\label{sec:exp}

We utilize  \mainAlg \ to obtain bounds on the approximation achieved by submodular maximization algorithms in practice. We show that   \textsc{Greedy} and other algorithms find solutions that approximate the optimal solution significantly better than  $1-1/e$ on a wide variety of real-world datasets and objectives. We also show that \mainAlg \ outperforms multiple benchmarks for deriving approximations for the solution found by \textsc{Greedy}. For all instances, we use $\mathcal S = \{S_1,S_2, \ldots,S_{20}\} \cup \{S_{25}, S_{30}, \ldots ,S_{50}\}$ as an input to \mainAlg, where $S_i$ is the greedy solution of size $i$. 

\subsection{Approximations for submodular maximization algorithms using \mainAlg}

We begin by evaluating the bounds derived by \mainAlg \ on the approximation achieved by four different submodular maximization algorithms.  
The goal here is
 not to provide a comprehensive comparison of submodular maximization algorithms, but to analyze  the approximations computed by \mainAlg.

\paragraph{Algorithms for submodular maximization.} 
The {\bf \textsc{Greedy}} algorithm obtains the optimal $1-1/e$ approximation \citep{NWF78} and is widely considered as the standard algorithm for monotone submodular maximization under a cardinality constraint. {\bf \textsc{Local search}} obtains a $1/2$ approximation guarantee \citep{NWF78} and is another widely used algorithm. {\bf \textsc{Lazier-than-lazy greedy}}, also called sample greedy, improves the running time of greedy by sampling a small subset of elements at each iteration \cite{MBKV15, buchbinder2015comparing}. {\bf \textsc{Random greedy}} handles submodular functions that are not necessarily monotonic by introducing randomization into the element selection step and obtains a $1-1/e$ approximation guarantee for monotone submoduar functions \cite{BFN14}. We provide further details on these algorithms in Appendix \ref{app:alg}. For randomized algorithms, we average the results over $5$ runs.

 \paragraph{Settings.} We examine the approximations computed by \mainAlg \ for these algorithms on $8$ different datasets and objectives. 
 Additional details can be found in Appendix \ref{app:large}. 

\begin{itemize}
\item {{\bf Influence maximization}: As in \cite{MBK16, BBS18}, we use {\bf Youtube social network} data \cite{YL15} and sample $n=1,045$ users from 50 large communities. We select $k$ users by maximizing coverage: $f(S) = |N_G(S)|$.}
\item {{\bf Car dispatch}: Our goal is to select the $k$ best locations to deploy drivers to cover the maximum number of pickups. As in \cite{BS18, KZK18}, we analyze $n=1,000$ locations of {\bf Uber pickups} \cite{FVE} and assign a weight $w_i$ to each neighborhood $n_i$ that is proportional to the number of trips in the neighborhood.} 
\item {\bf Influence maximization}: We use a {\bf citation network} of Physics collaborations \cite{LKF07} with $n=9,877$ authors (nodes) and 25,998 co-authorships (edges), and maximize $f(S) = |N_G(S)|$.
\item {{\bf Movie recommendation}: As in \cite{MBK16, MBNTC17, BBS18, BBS20}, we use the {\bf MovieLens dataset} \cite{HK16} of $n=3,706$ movies to recommend $k$ movies that have both good overall ratings and are highly rated by the most users.} 
\item {{\bf Facility Location}: As in \cite{LWD16, PST19}, we use facility location objective and the movie ranking matrix $[r_{ij}]$ from the {\bf MovieLens dataset} where $r_{ij}$ is user $j$'s ranking on movie $i$ to select $k$ movies to recommend from $N$ using $f(S) = \frac{1}{|N|} \sum_{i \in N} \max_{j\in U} r_{ij}$.}
\item {\bf Revenue maximization}: We use a revenue maximization objective from \cite{BBS20} on the {\bf CalTech Facebook Network dataset} \cite{TMP12} of 769 Facebook users $N$. 
\item{{\bf Feature selection}: We use the {\bf Adult Income dataset} \cite{BM98} and select a subset of features to predict income level $Y$. We extract 109 binary features as in \cite{KZK18} and use a joint entropy objective to select  features.}
\item {\bf Sensor placement}: As in \cite{KSG08, KMGG08, OY15}, we use the {\bf Berkeley Intel Lab dataset} which comprises of 54 sensors that collect temperature information. 
 We select $k$ sensors that 
 maximize entropy. 
\end{itemize}

 \begin{figure*}[t]
	\centering
	\begin{subfigure}[c]{\textwidth}
		\includegraphics[width=0.24\textwidth]{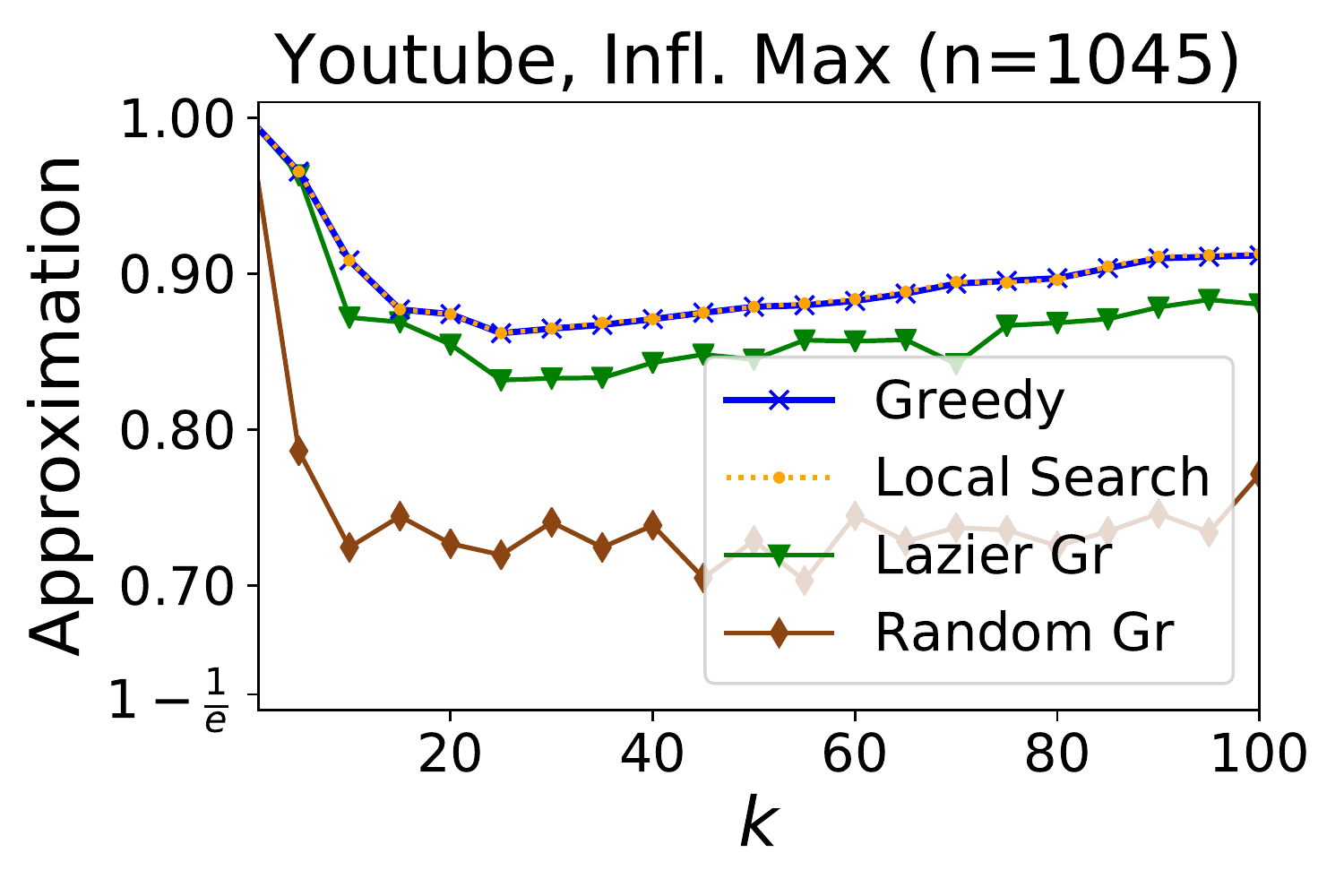}
		\includegraphics[width=0.24\textwidth]{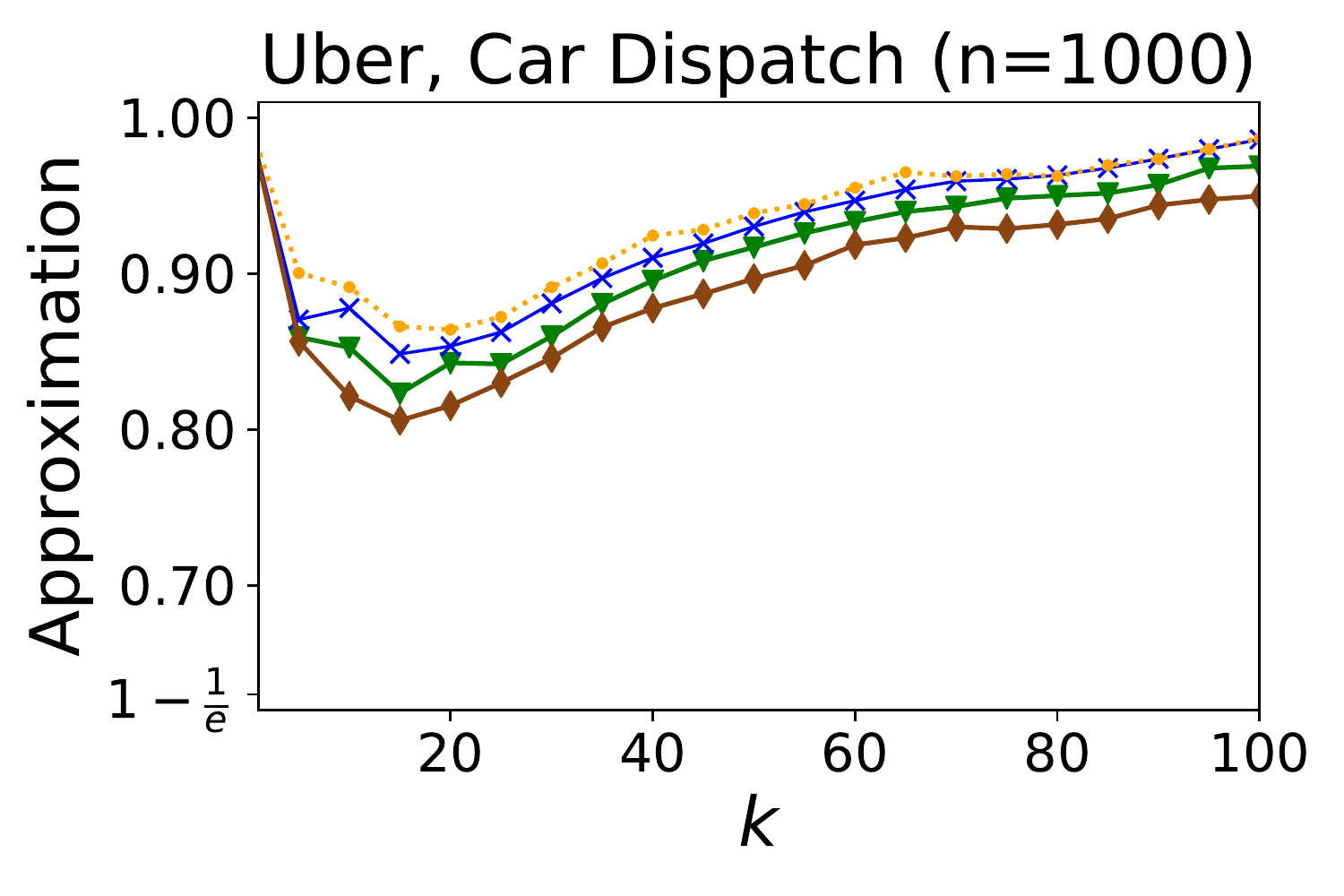}
		\includegraphics[width=0.24\textwidth]{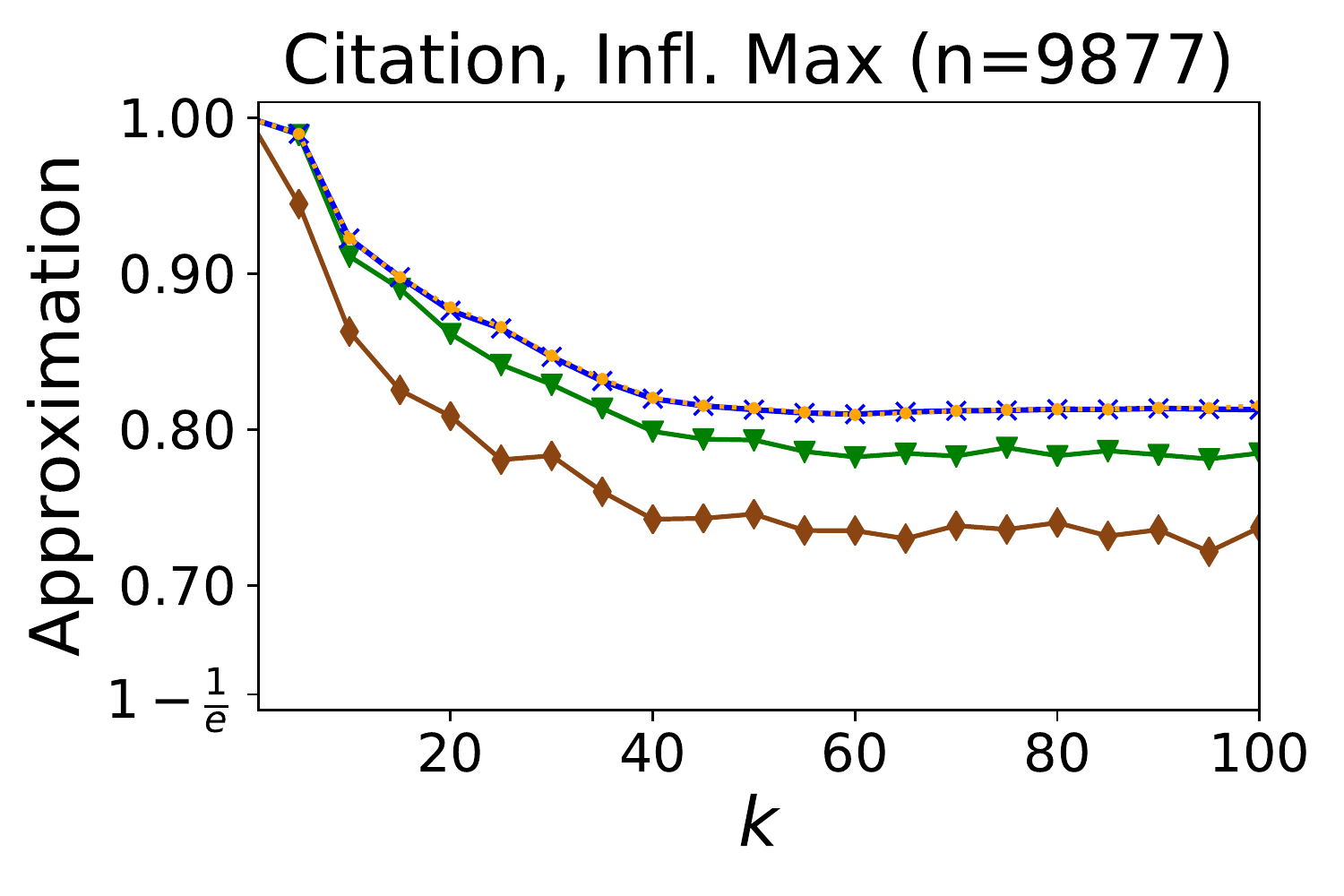}
		\includegraphics[width=0.24\textwidth]{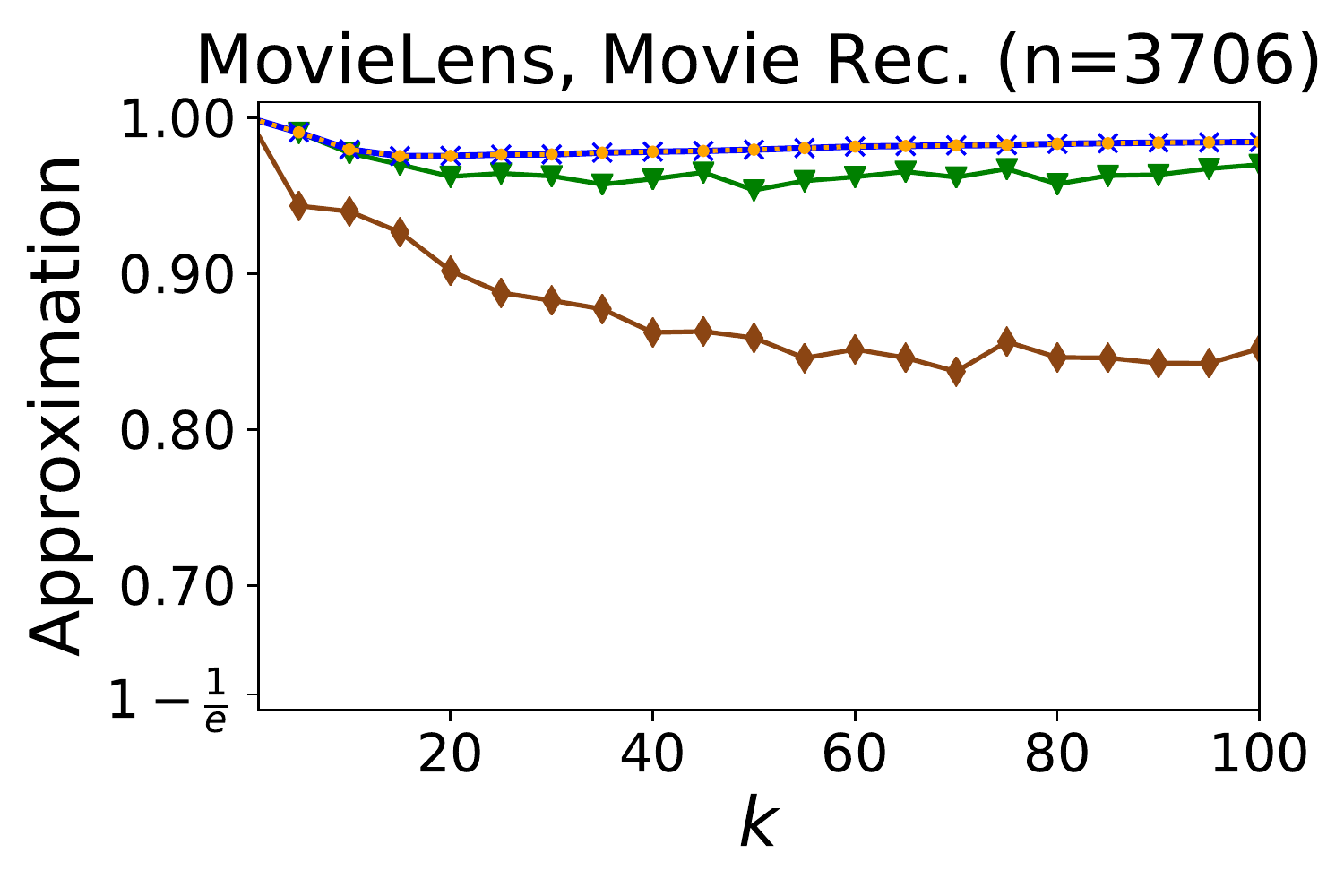}
	\end{subfigure} %
	\begin{subfigure}[c]{\textwidth}
		\includegraphics[width=0.24\textwidth]{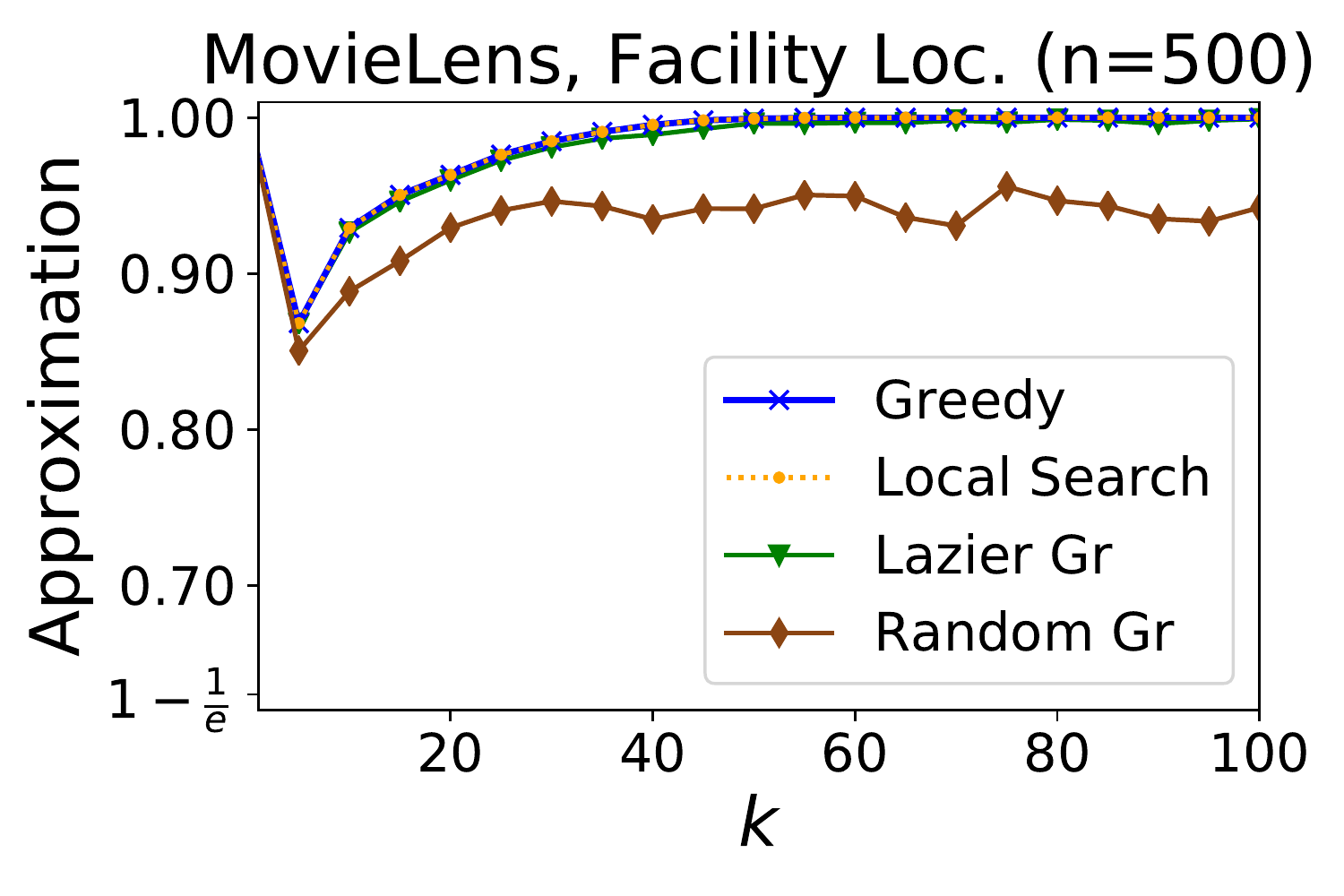}
		\includegraphics[width=0.24\textwidth]{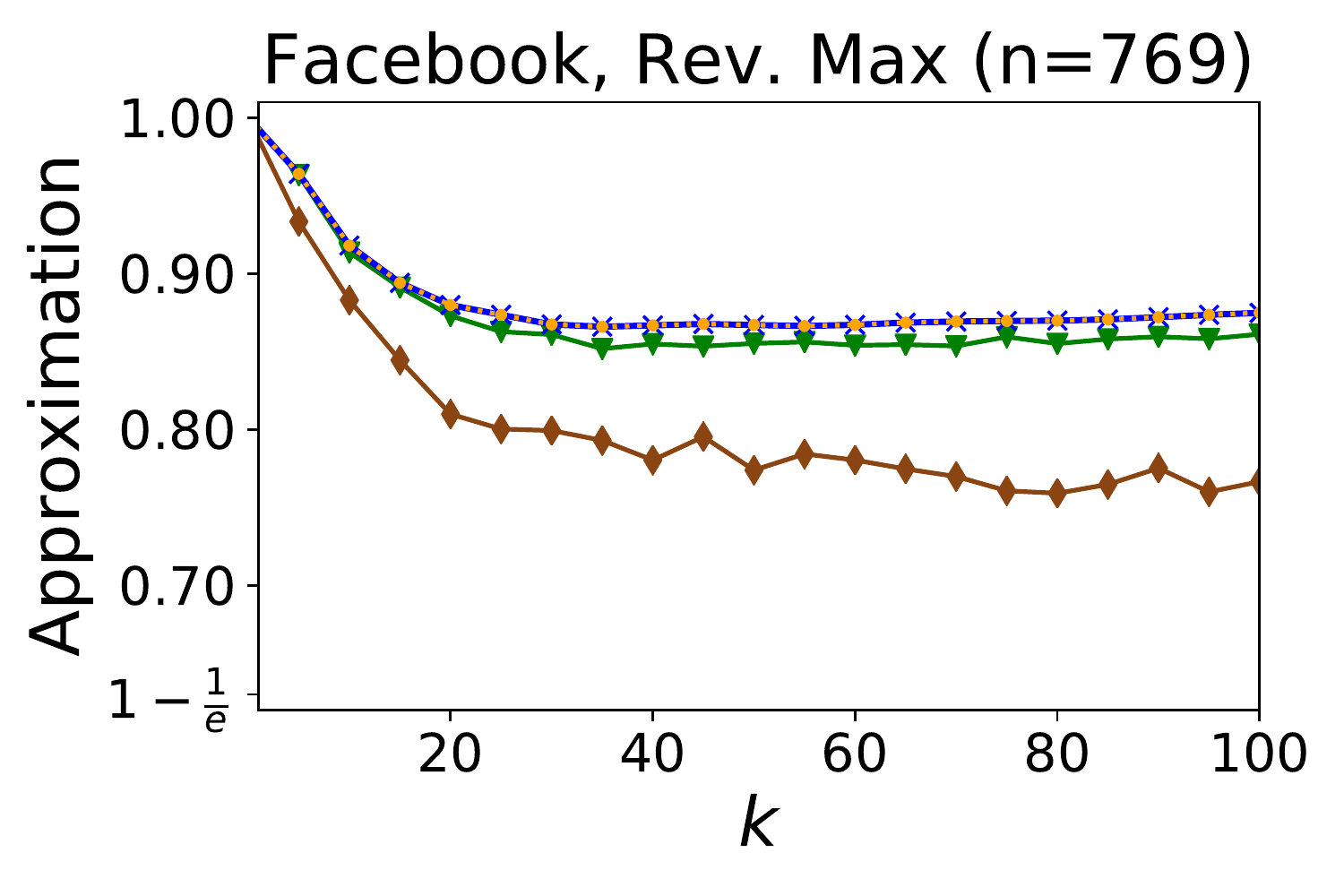}
		\includegraphics[width=0.24\textwidth]{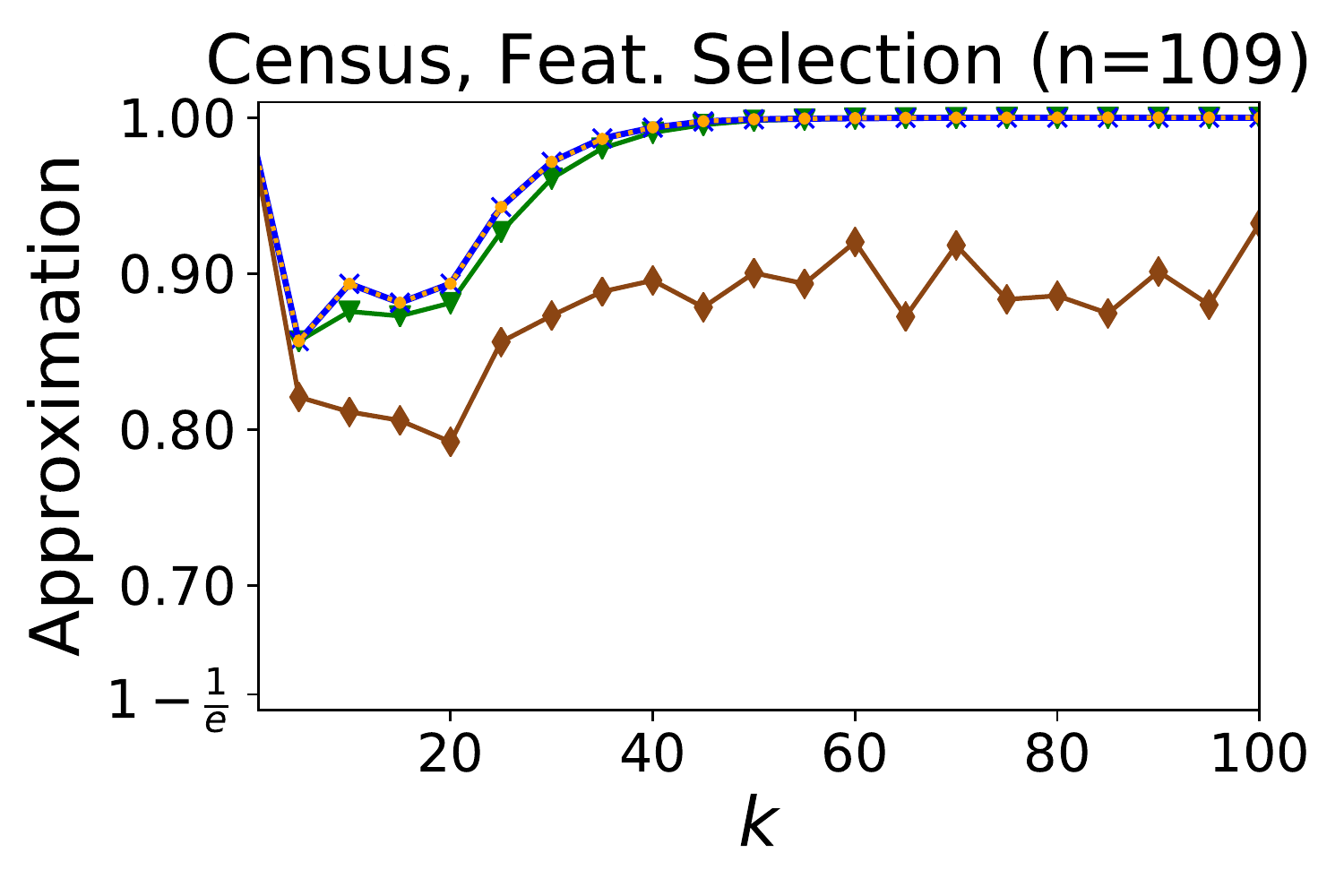}
		\includegraphics[width=0.24\textwidth]{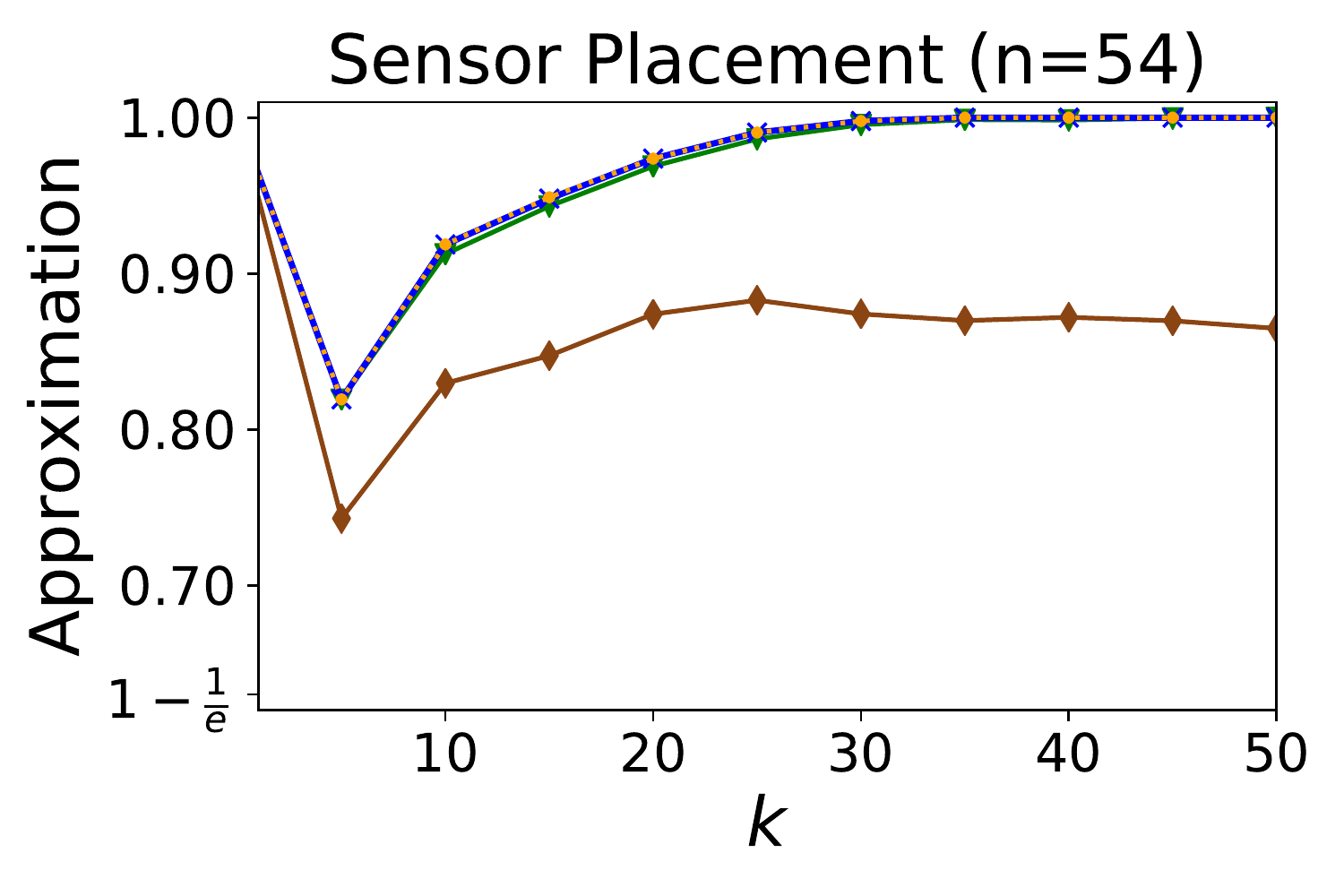}
	\end{subfigure} %
	\caption{Approximations computed by \mainAlg \ for the performance of four different submodular maximization algorithms.}
	\label{fig:algs}
\end{figure*}

\paragraph{Results.} In Figure \ref{fig:algs}, we see that \mainAlg \ computes bounds on the approximations achieved by all four algorithms that are significantly better than $1-1/e$. For \textsc{Greedy} and \textsc{Local search}, \mainAlg \ derives nearly identical approximations that are almost always over $0.85$. In many instances, such as movie recommendation, facility location, feature selection, and sensor placement, approximations are over $0.95$. The approximations given by \mainAlg \ for \textsc{Lazier-than-lazy greedy} are either identical to \textsc{Greedy} and \textsc{Local search}, or $0.02$ to $0.05$ worse. Even though \textsc{Random greedy} and \textsc{Greedy} have the same theoretical guarantee of $1-1/e$ for monotone submodular functions, the gap in their approximations on these instances is significant.

In most cases, the approximations obtained by \mainAlg \ follow a ``Nike-swoosh" shape as a function of constraint $k$. For $k=1$, the algorithms are either exactly or near-optimal; the lowest approximations are obtained for small values of $k$, and then approximations rebound and slowly increase as $k$ increases. For experiments with $n > 3000$ and the Facebook revenue maximization setting, the values of $k$ (up to $100$) are too small to observe this increase.





\subsection{\mainAlg  \ vs benchmarks for \textsc{Greedy} approximations}

 \begin{figure*}[t!]
	\centering
	\begin{subfigure}[b]{0.24\textwidth}
		\includegraphics[width=\textwidth]{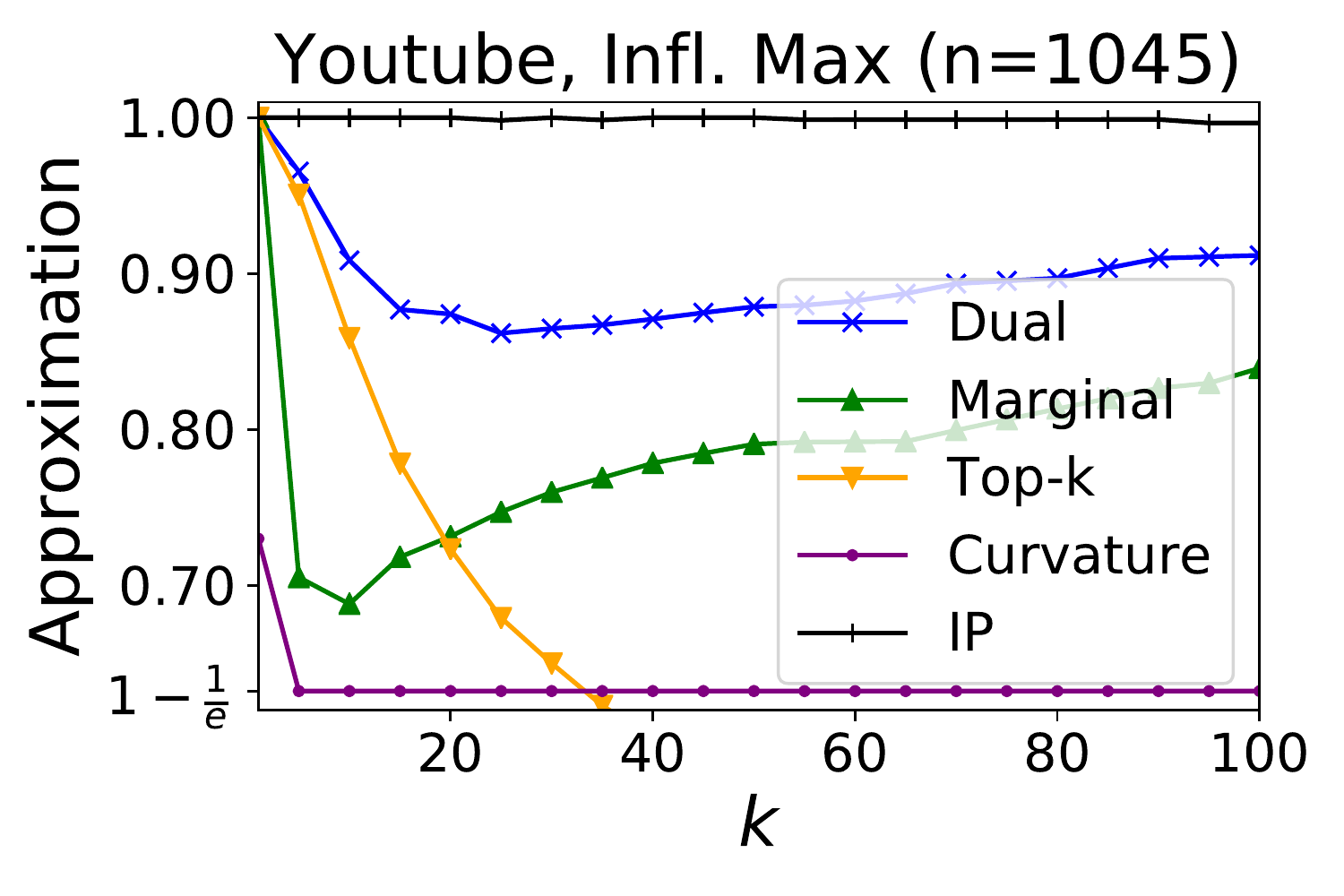}
	\end{subfigure} %
	\begin{subfigure}[b]{0.24\textwidth}    
		\includegraphics[width=\textwidth]{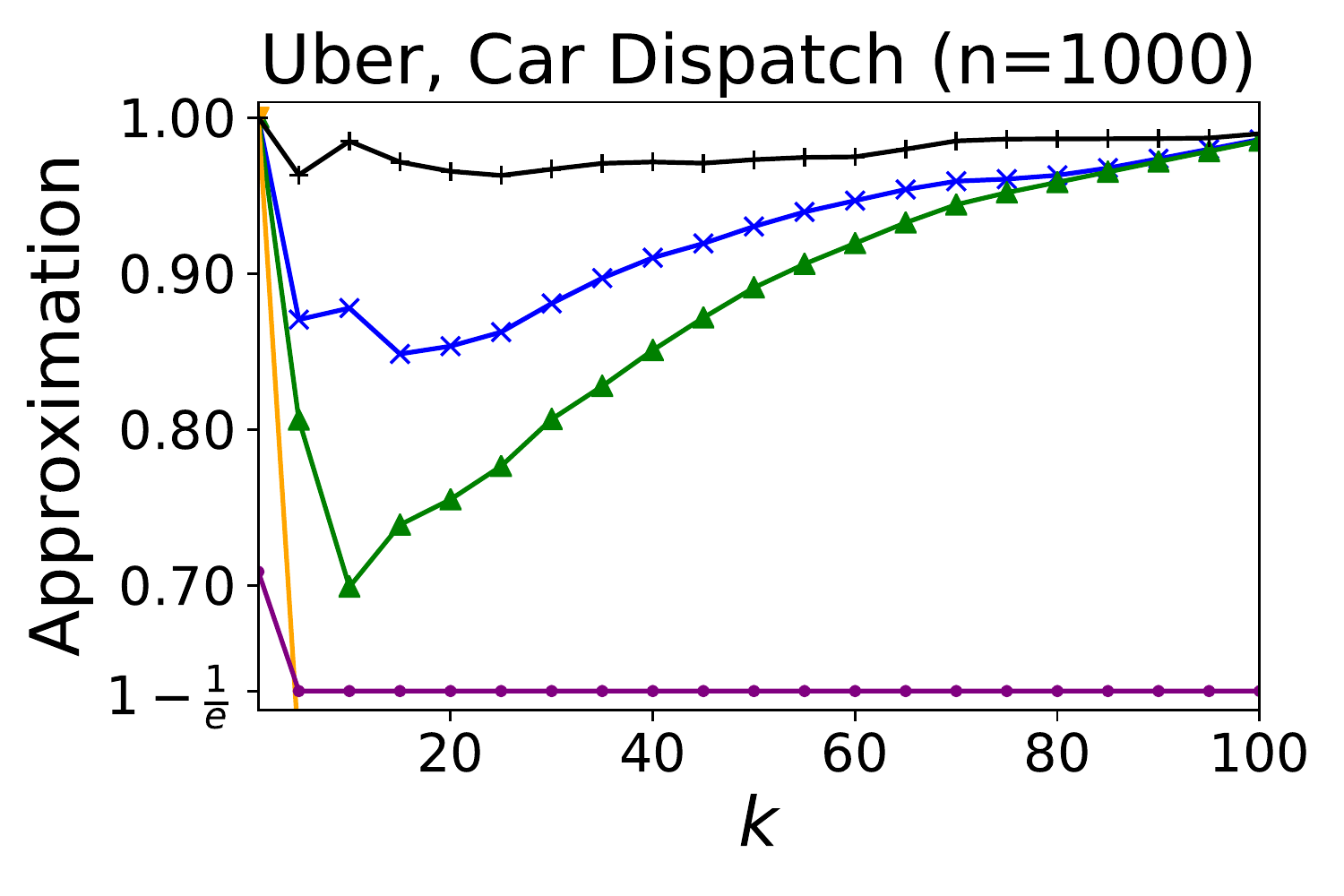}
	\end{subfigure} 
	\begin{subfigure}[b]{0.24\textwidth}    
		\includegraphics[width=\textwidth]{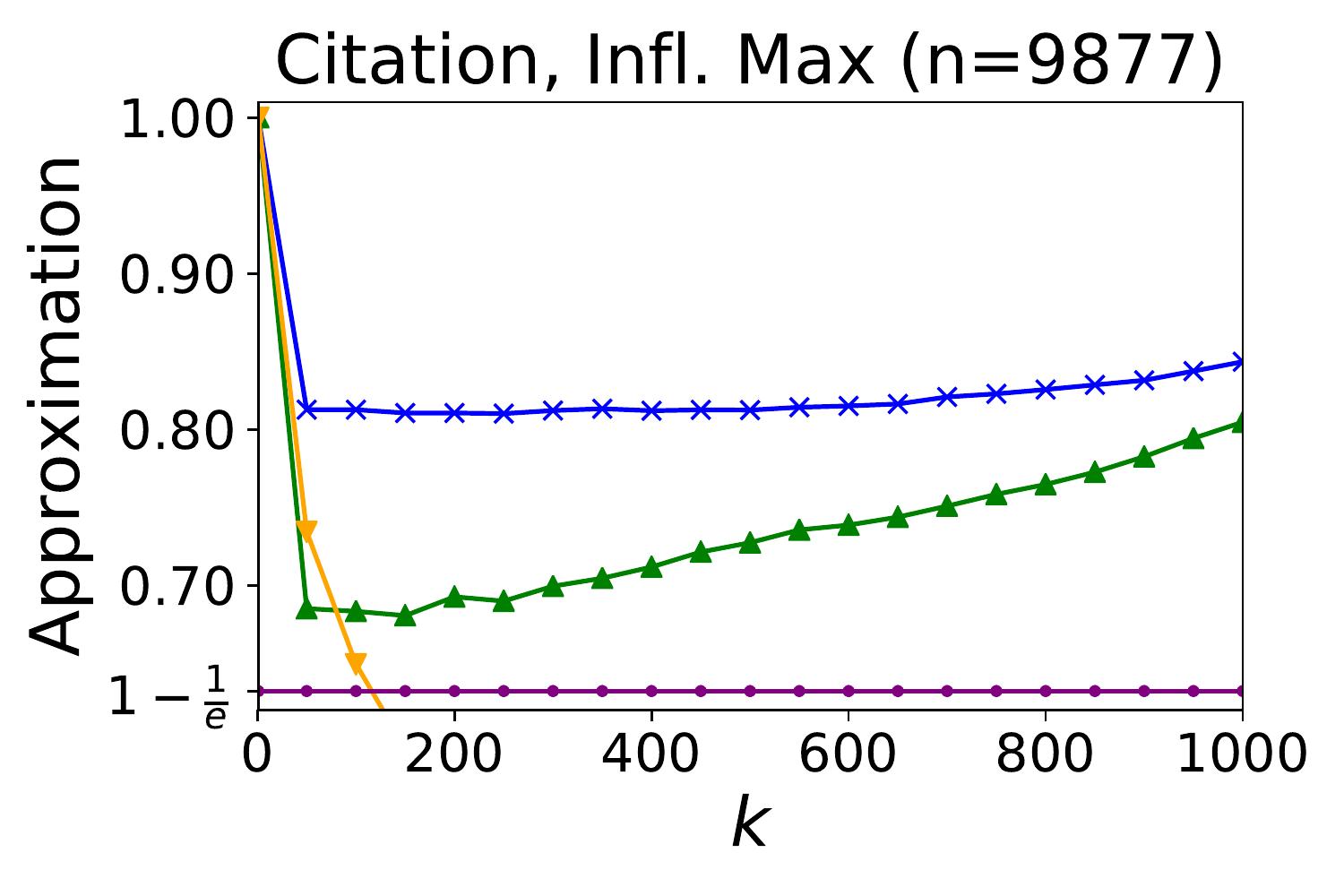}
	\end{subfigure} 
	\begin{subfigure}[b]{0.24\textwidth}    
		\includegraphics[width=\textwidth]{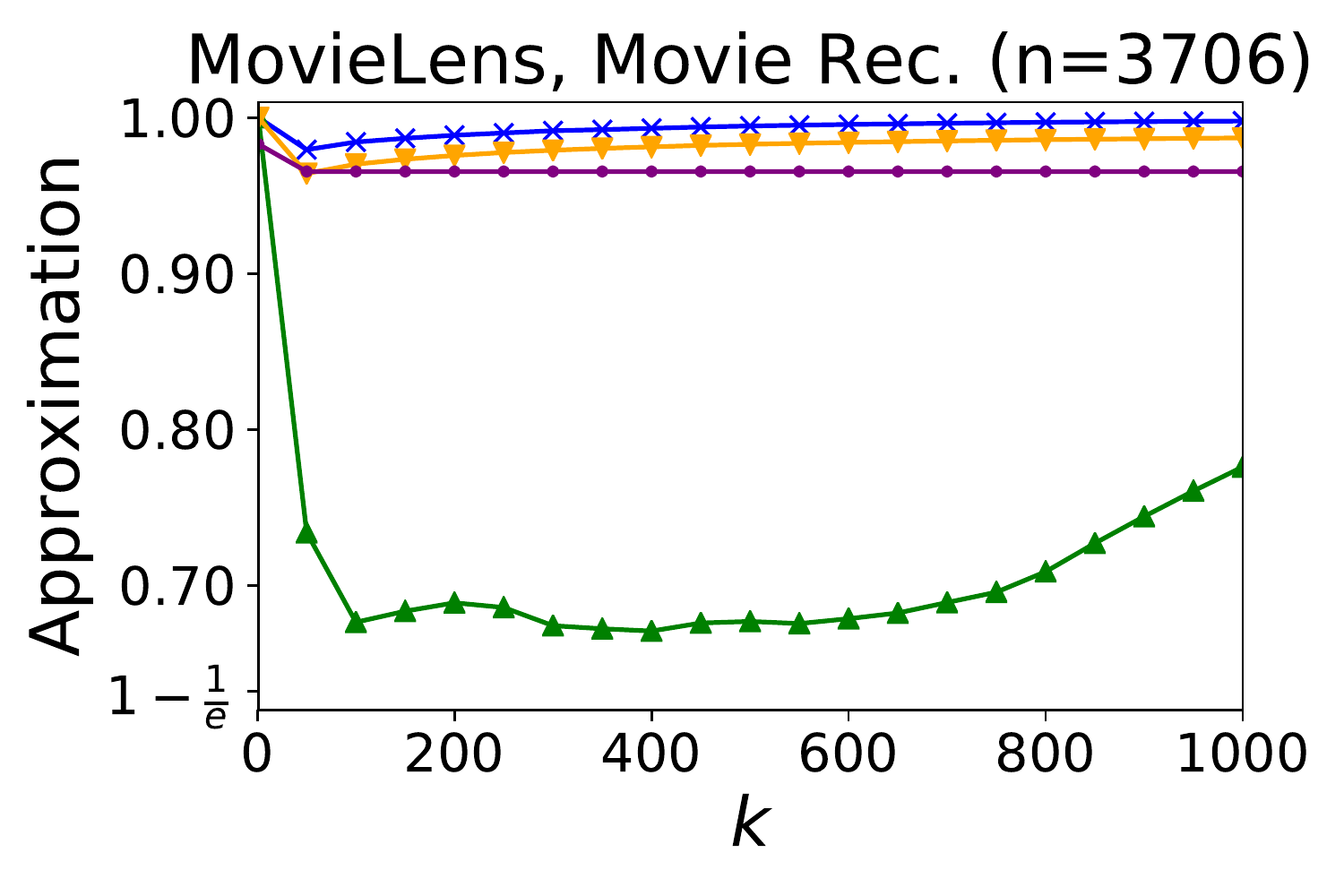}
	\end{subfigure} 
	\begin{subfigure}[b]{0.24\textwidth}
		\includegraphics[width=\textwidth]{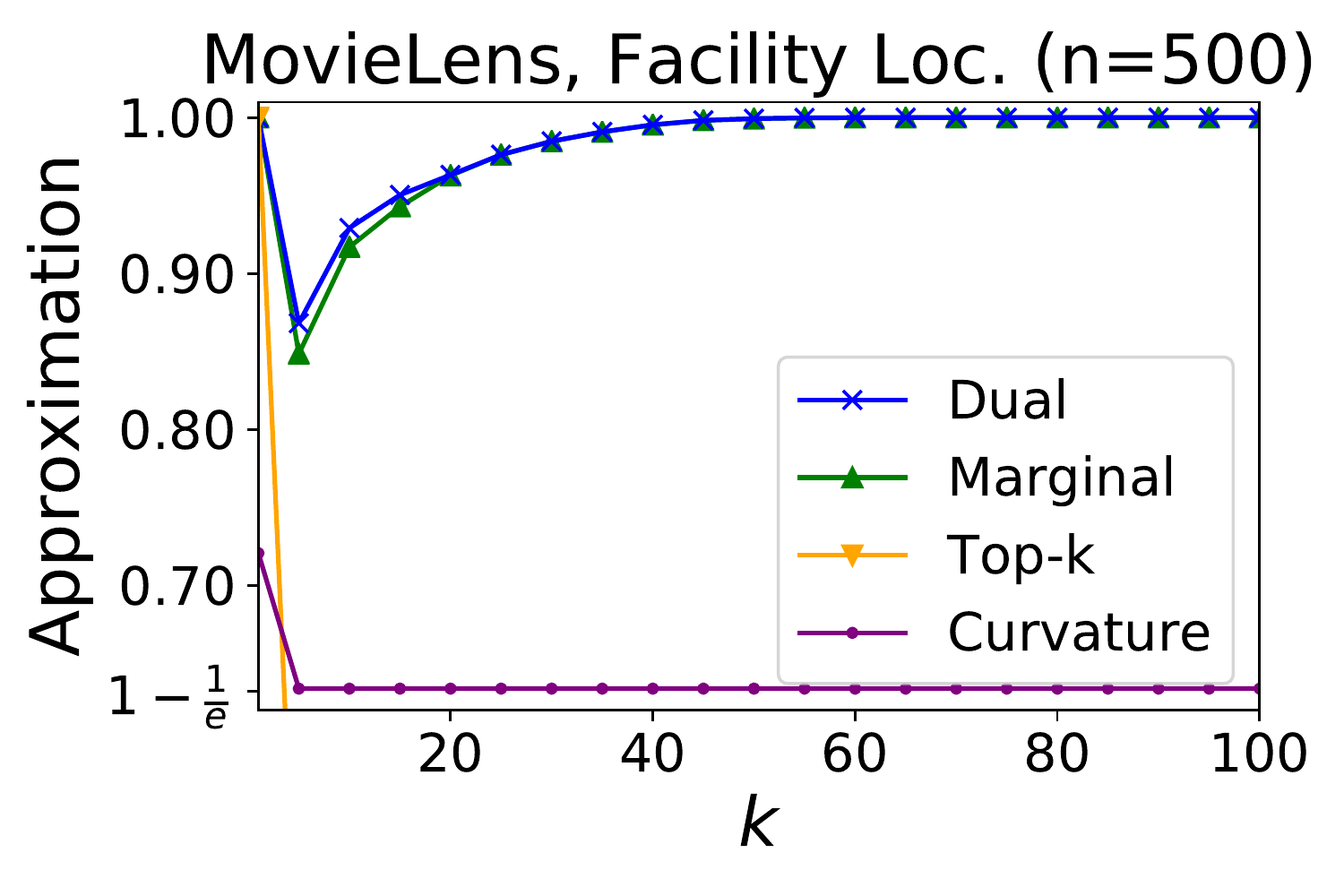}
	\end{subfigure} %
	\begin{subfigure}[b]{0.24\textwidth}    
		\includegraphics[width=\textwidth]{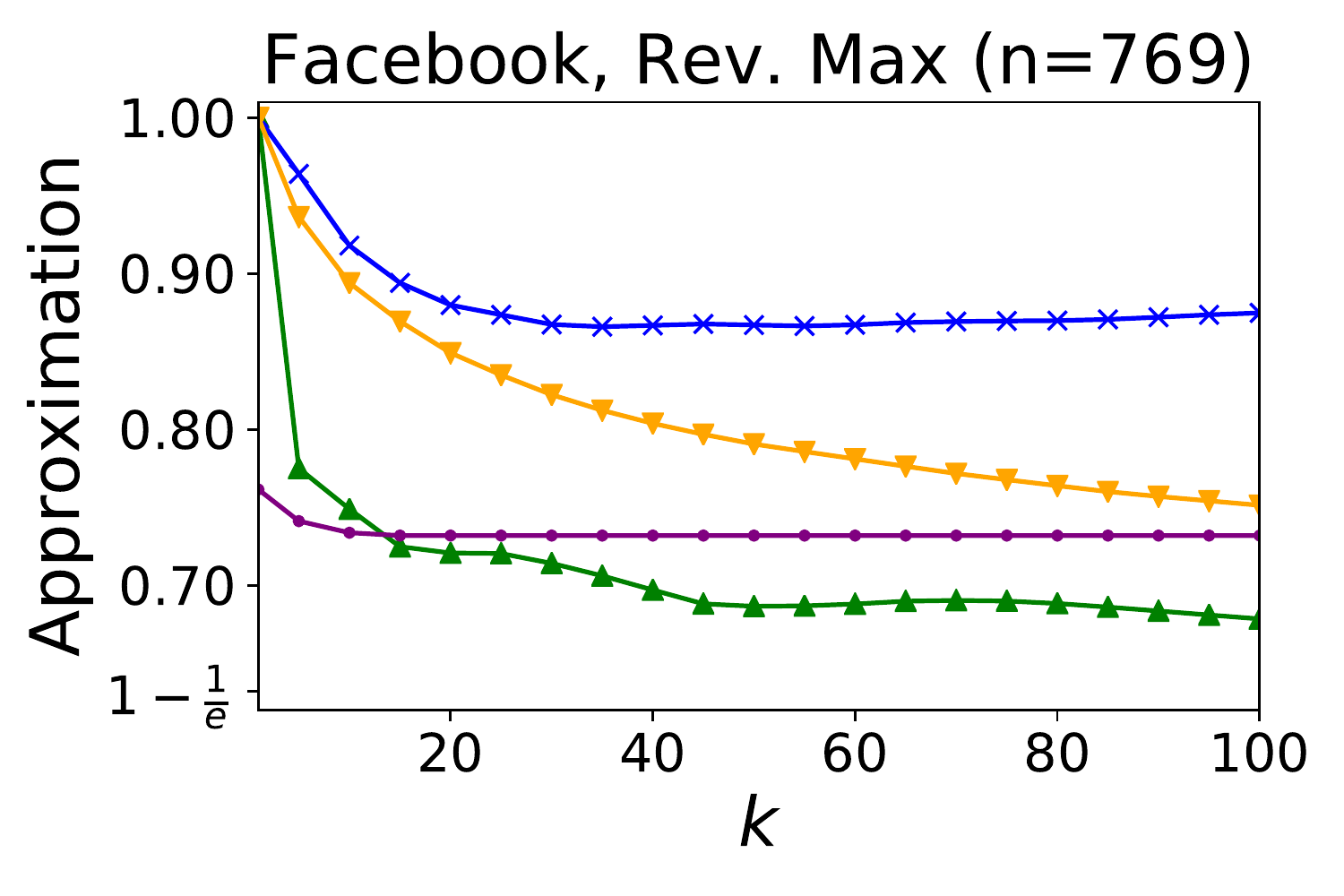} 
	\end{subfigure} 
	\begin{subfigure}[b]{0.24\textwidth}    
		\includegraphics[width=\textwidth]{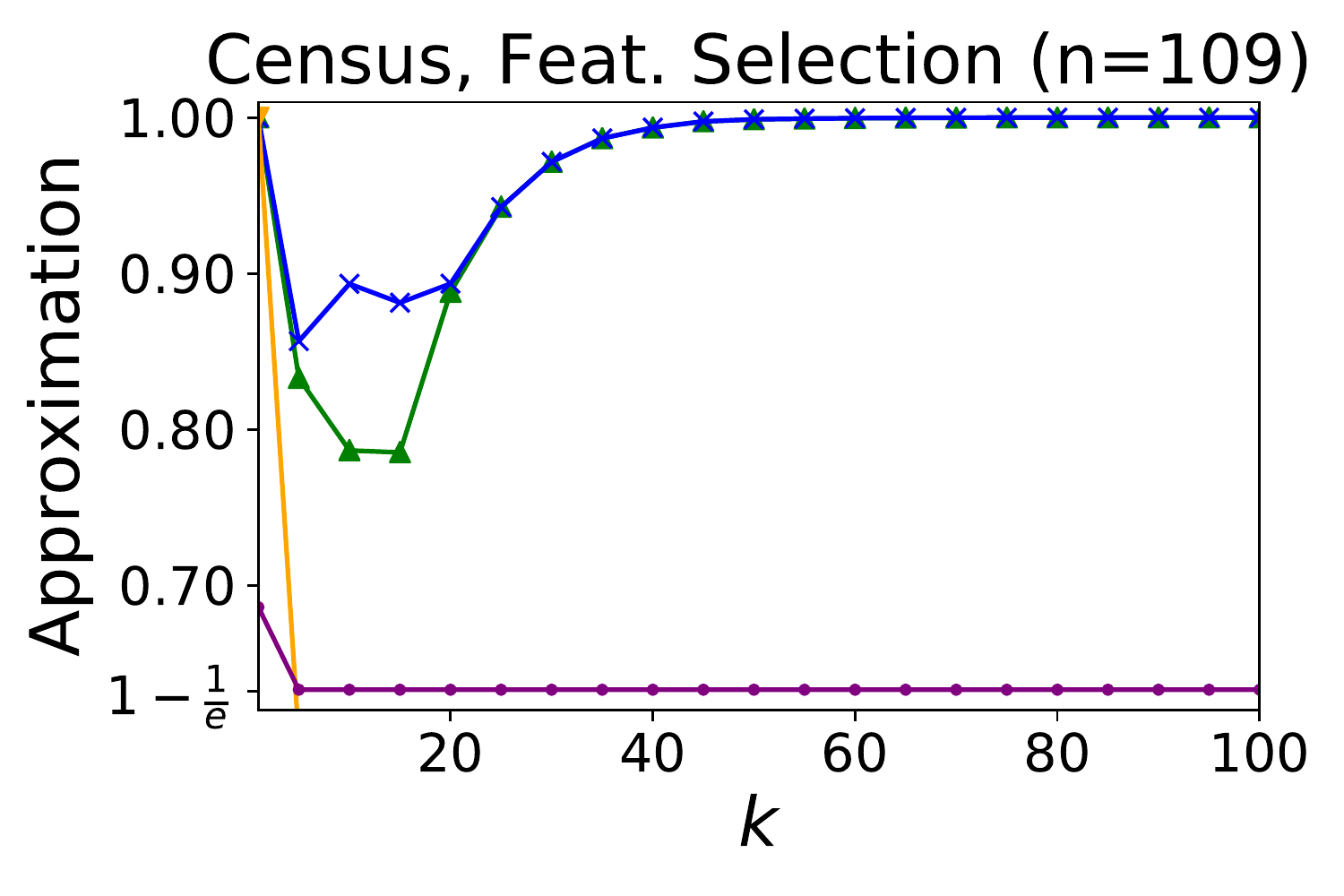}
	\end{subfigure} 
	\begin{subfigure}[b]{0.24\textwidth}    
		\includegraphics[width=\textwidth]{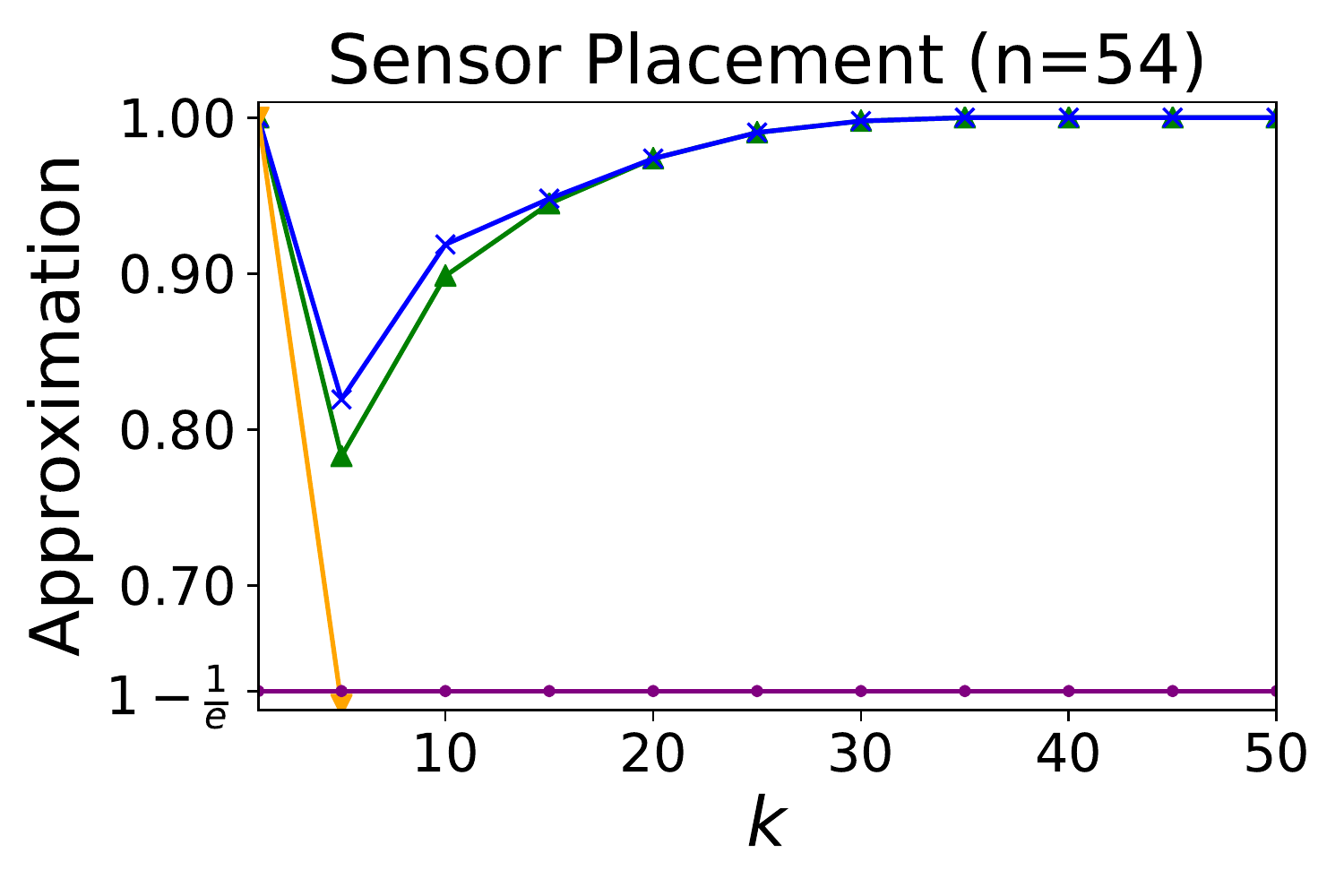} 
	\end{subfigure} 
	\caption{\textsc{Greedy} approximations computed by \mainAlg \ and multiple benchmarks on coverage objectives (top row) and submodular non-coverage objectives (bottom row).} 
	\label{fig:dual}
\end{figure*}

For the next set of experiments, we compare  \mainAlg \ to multiple benchmarks, which also compute approximations for the performance of submodular maximization algorithms. We fix a single  algorithm, \textsc{Greedy}, and compare the approximations found by \mainAlg \ and the benchmarks. 
We consider large instances, as well as small instances with $n \leq 20$ elements, where we can compute the curvature and sharpness benchmarks that require brute-force computation.

\paragraph{Benchmarks.} We consider the following benchmarks. 
\begin{itemize}
\item{{\bf Top-k.}} For a simple baseline, we upper bound \texttt{OPT} using the $k$ elements, $A$, with maximum singleton value $f(a)$: $\ubOPT_k =  \sum_{a\in A} f(a)$.

\item{{\bf Marginal.}} By using the value of \textsc{Greedy} solutions $S_i$ of size $i$ as well as  \textsc{Greedy} analysis techniques, we derive  the following more sophisticated bound:
 $$\OPT_k \leq  \frac{f(S_j) - (1-1/k)^{j-i}f(S_i)}{1-(1-1/k)^{j-i}}$$ for all $i < j$ (See Appendix \ref{app:bench1} for proof).
 We compute the minimum upper bound  over all $i < j \leq n$ pairs.
  
\item{{\bf Curvature.}} The curvature $c = 1 - \min_{S, a}\frac{f_S(a)}{f(a)}$ of a function  measures  how close $f$ is to additive  \citep{CC84}. It yields an improved $(1-e^{-c})/c$ approximation for \textsc{Greedy}.
 
\item {{\bf Sharpness.} The property of submodular sharpness was introduced by \citet*{PST19} as an explanation for the performance of \textsc{Greedy} in practice. It is the  analog of sharpness from continuous optimization \cite{L63} and assumes that any solution which differs significantly from the optimal solution has a substantially lower value.
On small instances, we compute the Dynamic Submodular Sharpness property of the function, which is the sharpness property that yields the best approximation. More details in Appendix \ref{app:bench_sharp}.}

\item{{\bf Integer Program (IP).} For the special case of coverage objectives, the problem can be formulated as an integer linear program. Integer programming is NP-complete and an optimal solution is not guaranteed to be found in polynomial time. By using IP on coverage functions with $n \leq 2000$ elements, we can find an optimal solution, and thus find the tight approximation achieved by \textsc{Greedy} solutions in these cases.}
\end{itemize}

\begin{figure*}
	\begin{minipage}[c]{0.49\linewidth}
		\begin{subfigure}[b]{0.48\textwidth}    
			\includegraphics[width=0.98\textwidth]{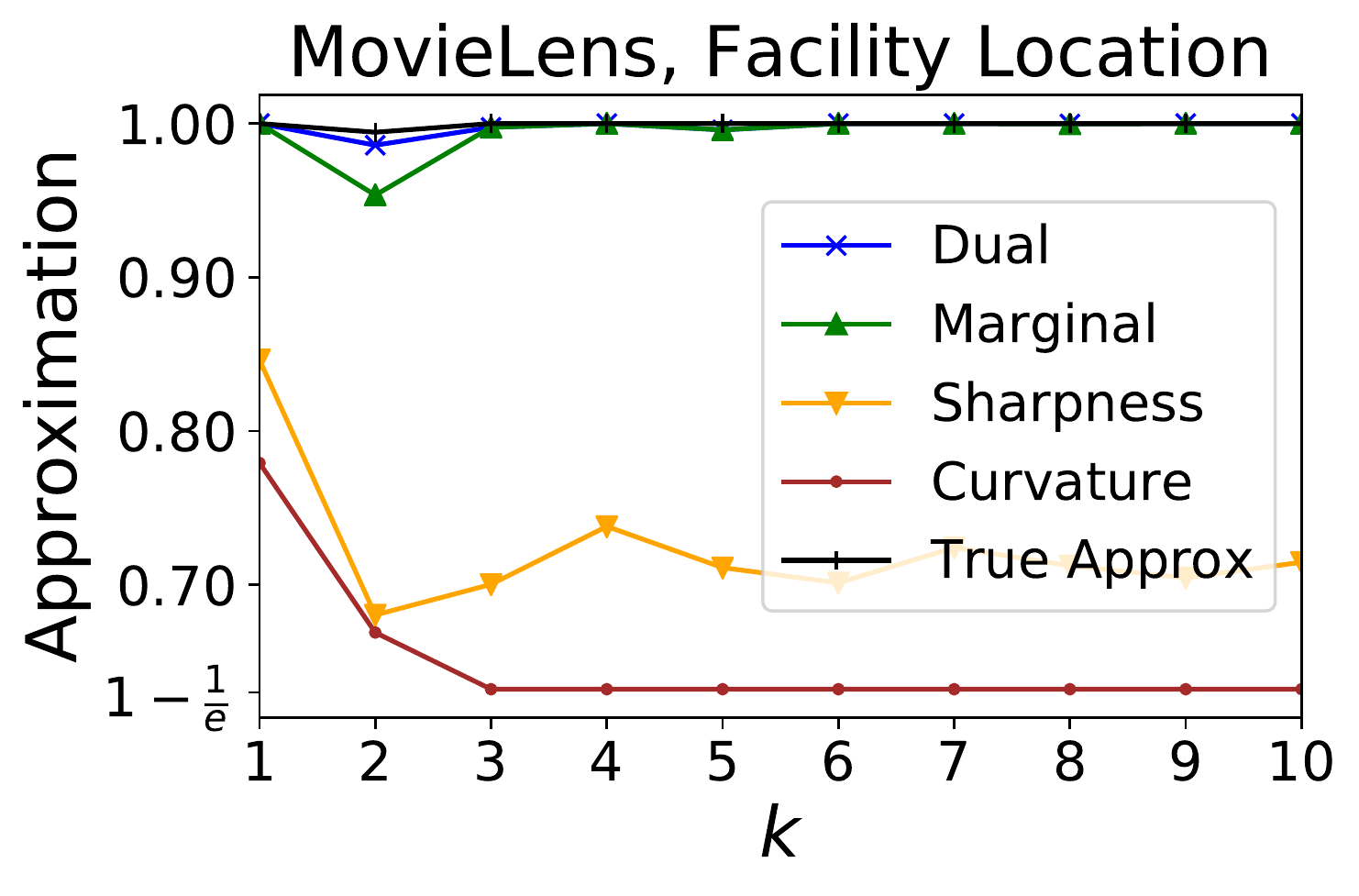} 
		\end{subfigure} 
		\begin{subfigure}[b]{0.48\textwidth}    
			\includegraphics[width=0.98\textwidth]{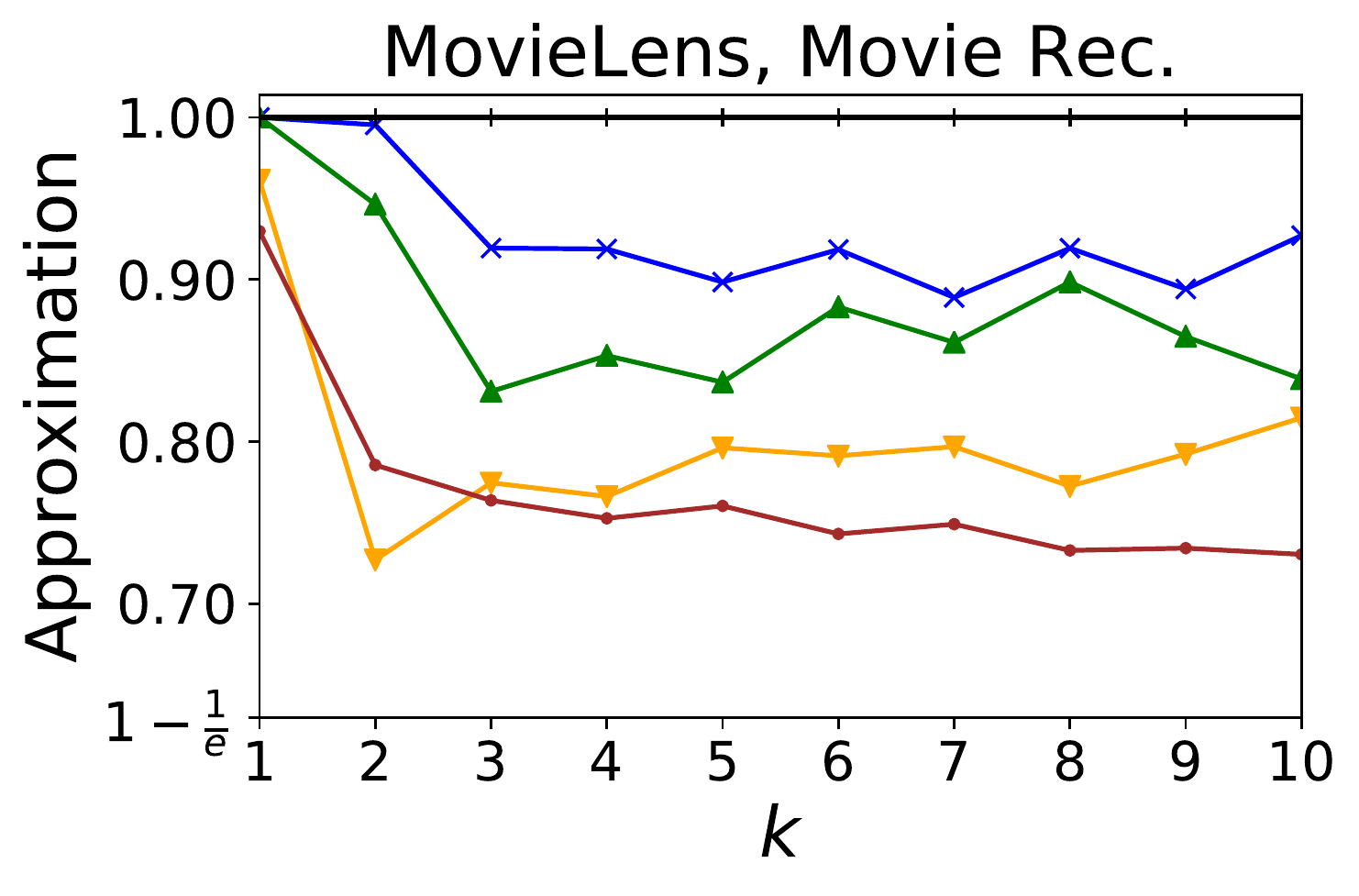}
		\end{subfigure} 
		\caption{\textsc{Greedy} approximations computed by \mainAlg \ and multiple benchmarks on small instances ($n\leq20$).}\label{fig:small}
	\end{minipage}%
	\hfill
	\begin{minipage}[c]{0.49\linewidth}
		\begin{subfigure}[b]{0.48\textwidth}
			\includegraphics[width=0.98\textwidth]{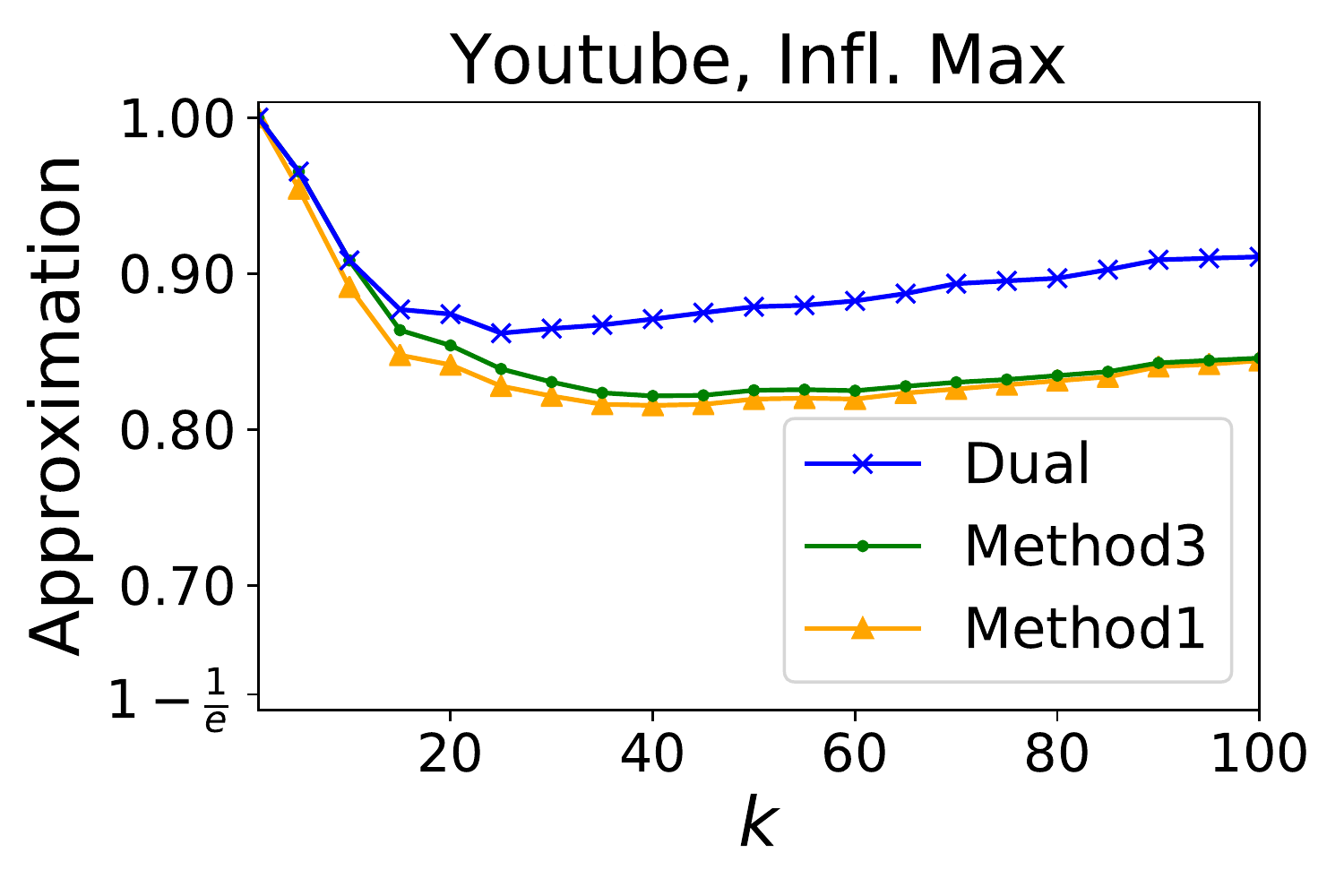}
		\end{subfigure} %
		\begin{subfigure}[b]{0.48\textwidth}    
			\includegraphics[width=0.98\textwidth]{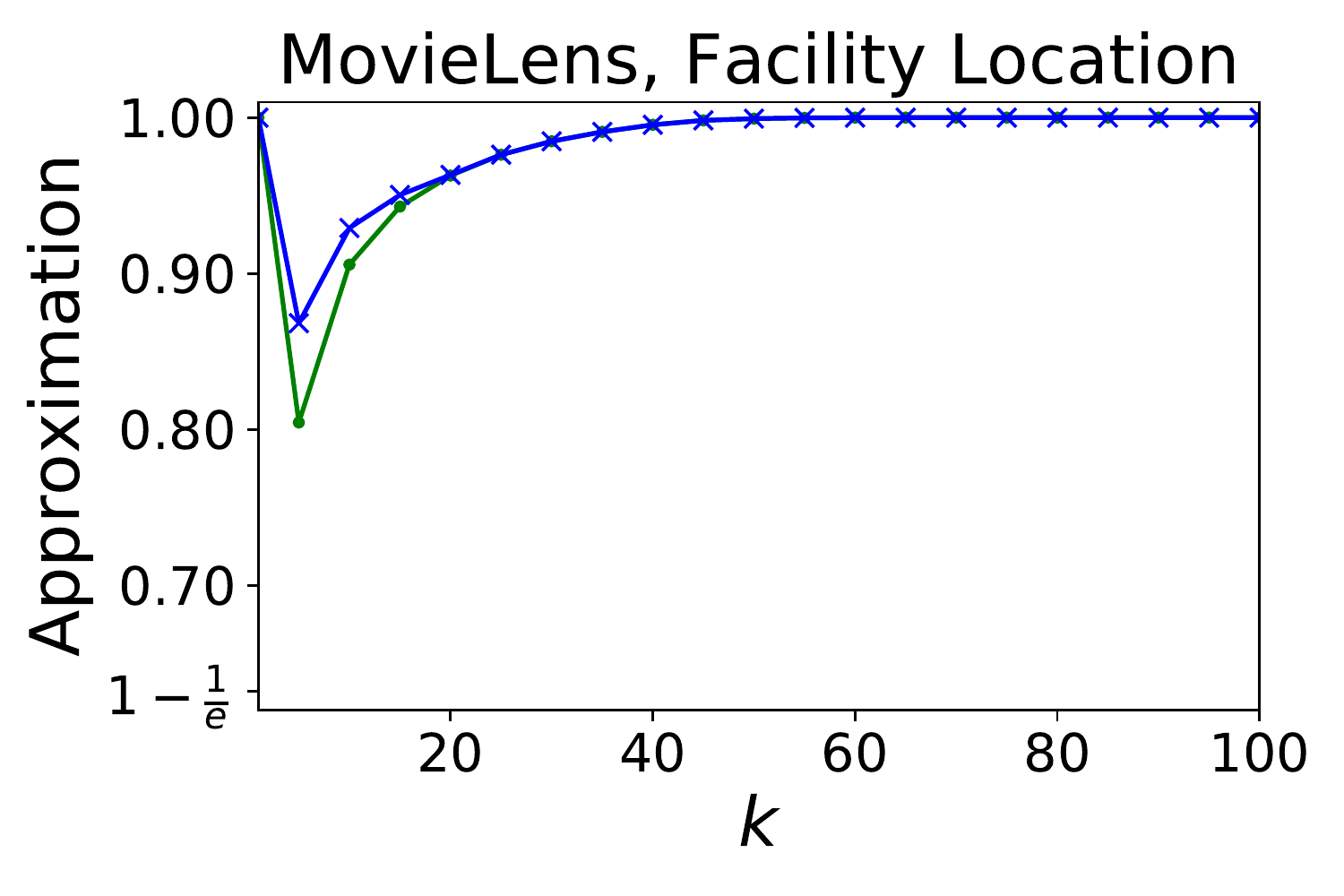}
		\end{subfigure} 
		\caption{\textsc{Greedy} approximations computed by Method~\ref{alg:coverage}, Method~\ref{alg:sm}, and \mainAlg \ on a coverage (left) and submodular objective (right).}\label{fig:alg}
	\end{minipage}
\end{figure*}

Unlike \mainAlg, \textsc{Top-k}, and \textsc{Marginal}, which upper bound $\OPT$,  \textsc{Curvature} and \textsc{Sharpness} each identify properties that guarantee a bound on the  \textsc{Greedy} approximation for any function that satisfies the properties. 
However, computing the parameters of these properties requires brute force computation and is computationally infeasible on large datasets.
Another benchmark is stability, which guarantees that \textsc{Greedy} finds the optimal solution if the instance is sufficiently stable, i.e., its  optimal solution remains optimal under small perturbations of the function \cite{CRV17}. However, our settings are not perturbation-stable because there are multiple near-optimal solutions. 
\subsubsection{Results on large instances}
We compare the approximations of \textsc{Greedy} found by \textsc{Dual} to those found by \textsc{Topk}, \textsc{Marginal} and \textsc{Curvature}. For \textsc{Curvature}, we compute an upper bound on the curvature parameter $c$ and the approximation $(1-e^{-c})/c$ (See Appendix \ref{app:bench_curv} for details). For coverage objectives, we additionally compute \textsc{IP} on datasets where $n \leq 2000$.

Figure \ref{fig:dual} shows that \textsc{Dual} consistently outperforms or matches the baselines. The exception is \textsc{IP}, which finds the tight approximation achieved by \textsc{Greedy} in two settings. 
The \textsc{Top-k}  and \textsc{Curvature} benchmarks perform poorly in most cases, which implies that most objectives are not close to additive. 
\textsc{Marginal} is the strongest general benchmark, but it is still significantly outperformed by \textsc{Dual} on most instances. For MovieLens movie recommendation and Facebook revenue maximization, where objectives are close to additive, \textsc{Top-k} and \textsc{Curvature} both outperform \textsc{Marginal}. For Youtube and Uber, \textsc{IP} shows that the tight approximation achieved by \textsc{Greedy} is above $0.99$, and between $0.95$ and $0.98$, respectively. Thus, even though \textsc{Dual} outperforms the other benchmarks, the results on Youtube and Uber settings indicate there remains a gap between the approximation computed by \textsc{Dual} and the tight approximation achieved by \textsc{Greedy}.

\subsubsection{Results on small instances} For small instances where we can exactly compute the sharpness and curvature parameters as well as $\OPT$ by brute-force, we follow the experimental setup from \cite{PST19}. For $k\in [1,10]$, we randomly choose $n=2k$ elements to comprise the ground set and analyze the result of \mainAlg \ versus benchmarks on objectives, facility location, and movie recommendation, from \cite{PST19} on the MovieLens dataset. More details in Appendix \ref{app:small}. 


In Figure \ref{fig:small}, we observe that \textsc{Dual}  yields the best approximations. For facility location, \textsc{Dual} and \textsc{Marginal}  show a near-optimal approximation while other benchmark approximations are near $1-1/e$ for larger $k$. For the movie recommendation objective, the gap between the different benchmarks is smaller. In both settings and for all $k$, \textsc{Greedy}  finds a near-optimal solution. 

\begin{table}[]
\centering
\resizebox{0.45\columnwidth}{!}{%
\begin{tabular}{l|lllll}
$k$  & \textsc{Marg.} & \textsc{Dual}   & \textsc{Opt} & \textsc{Curv.} & \textsc{Sharp.} \\ \hline
6  & 9.60 e-5 & 7.55 e-3 & 0.0212  & 0.176     & 0.448      \\
8  & 1.59 e-4 & 0.0129 & 0.249   & 3.55      & 9.00      \\
10 & 3.17 e-4 & 0.0246  & 3.43    & 74.7       &  187      
\end{tabular}
}
\caption{Average runtimes (in seconds) of benchmark methods on MovieLens facility location setting, where $n = 2k$.}
\label{tab:1}

\end{table}

We report benchmark runtimes in Table \ref{tab:1} for the facility location objective and find that \textsc{Curvature}, \textsc{Sharpness} and \texttt{OPT}, which all require brute-force computation, become exponentially slower as $k$ increases. At $k=10, n=20$, the average time to compute curvature approximation is 75 seconds while sharpness computation time is $187$ seconds. These methods are much slower than even brute-force computing \texttt{OPT} which takes $3.5$ seconds. While these benchmarks are not scalable, \textsc{Dual}, which is at least  $1000$ times faster than these two methods for the facility location objective when $k=10$, is scalable for larger datasets. 

\subsubsection{Comparison of proposed methods} We compare Method~\ref{alg:coverage}, Method~\ref{alg:sm}, and \textsc{Dual} on a coverage objective  and compare Method~\ref{alg:sm} and \textsc{Dual} on a non-coverage objective  in Figure \ref{fig:alg}. We observe that even Method~\ref{alg:coverage}, which employs the additive lower bound on the dual objective, finds approximations that are above $0.8$. This indicates that, unlike the primal objective, the dual objective is close to additive.
 By partitioning  the dual space (Method \ref{alg:sm}), a small improvement over  Method~\ref{alg:coverage} is achieved. Finally, by considering the upper bound on the optimal solution  for the functions $f_S$ for $S \in \S$ (\mainAlg), the approximations further improve. This improvement is minor for MovieLens, but can be around $0.1$ for some $k$ on YouTube. 

\newpage
\bibliography{main}

\begin{thebibliography}{60}
\providecommand{\natexlab}[1]{#1}
\providecommand{\url}[1]{\texttt{#1}}
\expandafter\ifx\csname urlstyle\endcsname\relax
  \providecommand{\doi}[1]{doi: #1}\else
  \providecommand{\doi}{doi: \begingroup \urlstyle{rm}\Url}\fi

\bibitem[Badanidiyuru et~al.(2014)Badanidiyuru, Mirzasoleiman, Karbasi, and
  Krause]{badanidiyuru2014streaming}
Badanidiyuru, A., Mirzasoleiman, B., Karbasi, A., and Krause, A.
\newblock Streaming submodular maximization: Massive data summarization on the
  fly.
\newblock In \emph{Proceedings of the 20th ACM SIGKDD international conference
  on Knowledge discovery and data mining}, pp.\  671--680, 2014.

\bibitem[Balkanski \& Singer(2018)Balkanski and Singer]{BS18}
Balkanski, E. and Singer, Y.
\newblock Approximation guarantees for adaptive sampling.
\newblock In \emph{International Conference on Machine Learning}, pp.\
  384--393, 2018.

\bibitem[Balkanski et~al.(2016)Balkanski, Rubinstein, and
  Singer]{balkanski2016power}
Balkanski, E., Rubinstein, A., and Singer, Y.
\newblock The power of optimization from samples.
\newblock In \emph{NIPS}, pp.\  4017--4025, 2016.

\bibitem[Balkanski et~al.(2017)Balkanski, Rubinstein, and
  Singer]{balkanski2017limitations}
Balkanski, E., Rubinstein, A., and Singer, Y.
\newblock The limitations of optimization from samples.
\newblock In \emph{Proceedings of the 49th Annual ACM SIGACT Symposium on
  Theory of Computing}, pp.\  1016--1027, 2017.

\bibitem[Balkanski et~al.(2018)Balkanski, Breuer, and Singer]{BBS18}
Balkanski, E., Breuer, A., and Singer, Y.
\newblock Non-monotone submodular maximization in exponentially fewer
  iterations.
\newblock In \emph{Advances in Neural Information Processing Systems}, pp.\
  2353--2364, 2018.

\bibitem[Balkanski et~al.(2019)Balkanski, Rubinstein, and
  Singer]{balkanski2019exponential}
Balkanski, E., Rubinstein, A., and Singer, Y.
\newblock An exponential speedup in parallel running time for submodular
  maximization without loss in approximation.
\newblock In \emph{Proceedings of the Thirtieth Annual ACM-SIAM Symposium on
  Discrete Algorithms}, pp.\  283--302. SIAM, 2019.

\bibitem[Barbosa et~al.(2016)Barbosa, Ene, Nguyen, and Ward]{barbosa2016new}
Barbosa, R. d.~P., Ene, A., Nguyen, H.~L., and Ward, J.
\newblock A new framework for distributed submodular maximization.
\newblock In \emph{2016 IEEE 57th Annual Symposium on Foundations of Computer
  Science (FOCS)}, pp.\  645--654. Ieee, 2016.

\bibitem[Blake \& Merz(1998)Blake and Merz]{BM98}
Blake, C.~L. and Merz, C.~J.
\newblock {UCI} machine learning repository, 1998.
\newblock URL \url{http://archive.ics.uci.edu/ml}.

\bibitem[Bogunovic et~al.(2017)Bogunovic, Mitrovi{\'c}, Scarlett, and
  Cevher]{bogunovic2017robust}
Bogunovic, I., Mitrovi{\'c}, S., Scarlett, J., and Cevher, V.
\newblock Robust submodular maximization: A non-uniform partitioning approach.
\newblock In \emph{International Conference on Machine Learning}, pp.\
  508--516. PMLR, 2017.

\bibitem[Breuer et~al.(2020)Breuer, Balkanski, and Singer]{BBS20}
Breuer, A., Balkanski, E., and Singer, Y.
\newblock The {FAST} algorithm for submodular maximization.
\newblock In \emph{Proceedings of the 37th International Conference on Machine
  Learning}, pp.\  1134--1143, 2020.

\bibitem[Buchbinder et~al.(2014)Buchbinder, Feldman, Naor, and Schwartz]{BFN14}
Buchbinder, N., Feldman, M., Naor, J., and Schwartz, R.
\newblock Submodular maximization with cardinality constraints.
\newblock In \emph{Proceedings of the twenty-fifth annual ACM-SIAM symposium on
  Discrete algorithms}, pp.\  1433--1452. SIAM, 2014.

\bibitem[Buchbinder et~al.(2015)Buchbinder, Feldman, and
  Schwartz]{buchbinder2015comparing}
Buchbinder, N., Feldman, M., and Schwartz, R.
\newblock Comparing apples and oranges: Query tradeoff in submodular
  maximization.
\newblock In \emph{SODA}, number CONF, pp.\  1149--1168, 2015.

\bibitem[Calinescu et~al.(2007)Calinescu, Chekuri, P{\'a}l, and
  Vondr{\'a}k]{calinescu2007maximizing}
Calinescu, G., Chekuri, C., P{\'a}l, M., and Vondr{\'a}k, J.
\newblock Maximizing a submodular set function subject to a matroid constraint.
\newblock In \emph{International Conference on Integer Programming and
  Combinatorial Optimization}, pp.\  182--196. Springer, 2007.

\bibitem[Chatziafratis et~al.(2017)Chatziafratis, Roughgarden, and
  Vondr{\'a}k]{CRV17}
Chatziafratis, V., Roughgarden, T., and Vondr{\'a}k, J.
\newblock Stability and recovery for independence systems.
\newblock \emph{arXiv preprint arXiv:1705.00127}, 2017.

\bibitem[Chekuri \& Quanrud(2019)Chekuri and Quanrud]{chekuri2019}
Chekuri, C. and Quanrud, K.
\newblock Submodular function maximization in parallel via the multilinear
  relaxation.
\newblock In \emph{Proceedings of the Thirtieth Annual ACM-SIAM Symposium on
  Discrete Algorithms}, pp.\  303--322. SIAM, 2019.

\bibitem[Chekuri et~al.(2015)Chekuri, Gupta, and Quanrud]{chekuri2015streaming}
Chekuri, C., Gupta, S., and Quanrud, K.
\newblock Streaming algorithms for submodular function maximization.
\newblock In \emph{International Colloquium on Automata, Languages, and
  Programming}, pp.\  318--330. Springer, 2015.

\bibitem[Chen et~al.(2019)Chen, Feldman, and Karbasi]{chen2018}
Chen, L., Feldman, M., and Karbasi, A.
\newblock Unconstrained submodular maximization with constant adaptive
  complexity.
\newblock \emph{STOC}, 2019.

\bibitem[Chen et~al.(2017)Chen, Lucier, Singer, and Syrgkanis]{chen2017robust}
Chen, R., Lucier, B., Singer, Y., and Syrgkanis, V.
\newblock Robust optimization for non-convex objectives.
\newblock \emph{arXiv preprint arXiv:1707.01047}, 2017.

\bibitem[Conforti \& Cornu{\'e}jols(1984)Conforti and Cornu{\'e}jols]{CC84}
Conforti, M. and Cornu{\'e}jols, G.
\newblock Submodular set functions, matroids and the greedy algorithm: tight
  worst-case bounds and some generalizations of the rado-edmonds theorem.
\newblock \emph{Discrete applied mathematics}, 7\penalty0 (3):\penalty0
  251--274, 1984.

\bibitem[Das \& Kempe(2011{\natexlab{a}})Das and Kempe]{das2011}
Das, A. and Kempe, D.
\newblock Submodular meets spectral: greedy algorithms for subset selection,
  sparse approximation and dictionary selection.
\newblock In \emph{Proceedings of the 28th International Conference on
  International Conference on Machine Learning}, pp.\  1057--1064. Omnipress,
  2011{\natexlab{a}}.

\bibitem[Das \& Kempe(2011{\natexlab{b}})Das and Kempe]{das2011submodular}
Das, A. and Kempe, D.
\newblock Submodular meets spectral: Greedy algorithms for subset selection,
  sparse approximation and dictionary selection.
\newblock \emph{arXiv preprint arXiv:1102.3975}, 2011{\natexlab{b}}.

\bibitem[Elenberg et~al.(2018)Elenberg, Khanna, Dimakis, Negahban,
  et~al.]{elenberg2018}
Elenberg, E.~R., Khanna, R., Dimakis, A.~G., Negahban, S., et~al.
\newblock Restricted strong convexity implies weak submodularity.
\newblock \emph{The Annals of Statistics}, 46\penalty0 (6B):\penalty0
  3539--3568, 2018.

\bibitem[Ene et~al.(2019)Ene, Nguyen, and Vladu]{ene2018}
Ene, A., Nguyen, H.~L., and Vladu, A.
\newblock Submodular maximization with matroid and packing constraints in
  parallel.
\newblock \emph{STOC}, 2019.

\bibitem[Fahrbach et~al.(2019)Fahrbach, Mirrokni, and
  Zadimoghaddam]{farbach2019}
Fahrbach, M., Mirrokni, V., and Zadimoghaddam, M.
\newblock Submodular maximization with optimal approximation, adaptivity and
  query complexity.
\newblock \emph{SODA}, 2019.

\bibitem[Feldman et~al.(2018)Feldman, Karbasi, and Kazemi]{FKK18}
Feldman, M., Karbasi, A., and Kazemi, E.
\newblock Do less, get more: Streaming submodular maximization with
  subsampling.
\newblock In Bengio, S., Wallach, H., Larochelle, H., Grauman, K.,
  Cesa-Bianchi, N., and Garnett, R. (eds.), \emph{Advances in Neural
  Information Processing Systems}, volume~31, pp.\  732--742. Curran
  Associates, Inc., 2018.
\newblock URL
  \url{https://proceedings.neurips.cc/paper/2018/file/d1f255a373a3cef72e03aa9d980c7eca-Paper.pdf}.

\bibitem[FiveThirtyEight(2019)]{FVE}
FiveThirtyEight.
\newblock Kaggle, 2019.
\newblock URL
  \url{https://www.kaggle.com/fivethirtyeight/uber-pickups-in-new-york-city}.

\bibitem[Guestrin et~al.(2005)Guestrin, Krause, and Singh]{guestrin2005near}
Guestrin, C., Krause, A., and Singh, A.~P.
\newblock Near-optimal sensor placements in gaussian processes.
\newblock In \emph{Proceedings of the 22nd international conference on Machine
  learning}, pp.\  265--272, 2005.

\bibitem[Harper \& Konstan(2015)Harper and Konstan]{HK16}
Harper, F.~M. and Konstan, J.~A.
\newblock The movielens datasets: History and context.
\newblock \emph{Acm transactions on interactive intelligent systems (tiis)},
  5\penalty0 (4):\penalty0 1--19, 2015.

\bibitem[Hassidim \& Singer(2017)Hassidim and Singer]{hassidim2017submodular}
Hassidim, A. and Singer, Y.
\newblock Submodular optimization under noise.
\newblock In \emph{Conference on Learning Theory}, pp.\  1069--1122. PMLR,
  2017.

\bibitem[Hassidim \& Singer(2018)Hassidim and Singer]{hassidim2018optimization}
Hassidim, A. and Singer, Y.
\newblock Optimization for approximate submodularity.
\newblock In \emph{Proceedings of the 32nd International Conference on Neural
  Information Processing Systems}, pp.\  394--405, 2018.

\bibitem[Horel \& Singer(2016)Horel and Singer]{horel2016maximization}
Horel, T. and Singer, Y.
\newblock Maximization of approximately submodular functions.
\newblock In \emph{NIPS}, volume~16, pp.\  3045--3053, 2016.

\bibitem[Kazemi et~al.(2018)Kazemi, Zadimoghaddam, and Karbasi]{KZK18}
Kazemi, E., Zadimoghaddam, M., and Karbasi, A.
\newblock Scalable deletion-robust submodular maximization: Data summarization
  with privacy and fairness constraints.
\newblock In \emph{International conference on machine learning}, pp.\
  2544--2553, 2018.

\bibitem[Kazemi et~al.(2019)Kazemi, Mitrovic, Zadimoghaddam, Lattanzi, and
  Karbasi]{kazemi2019submodular}
Kazemi, E., Mitrovic, M., Zadimoghaddam, M., Lattanzi, S., and Karbasi, A.
\newblock Submodular streaming in all its glory: Tight approximation, minimum
  memory and low adaptive complexity.
\newblock In \emph{International Conference on Machine Learning}, pp.\
  3311--3320. PMLR, 2019.

\bibitem[Kempe et~al.(2003)Kempe, Kleinberg, and Tardos]{kempe2003maximizing}
Kempe, D., Kleinberg, J., and Tardos, {\'E}.
\newblock Maximizing the spread of influence through a social network.
\newblock In \emph{KDD}, 2003.

\bibitem[Krause et~al.(2008{\natexlab{a}})Krause, McMahan, Guestrin, and
  Gupta]{KMGG08}
Krause, A., McMahan, H.~B., Guestrin, C., and Gupta, A.
\newblock Robust submodular observation selection.
\newblock \emph{Journal of Machine Learning Research}, 9\penalty0
  (93):\penalty0 2761--2801, 2008{\natexlab{a}}.
\newblock URL \url{http://jmlr.org/papers/v9/krause08b.html}.

\bibitem[Krause et~al.(2008{\natexlab{b}})Krause, Singh, and Guestrin]{KSG08}
Krause, A., Singh, A., and Guestrin, C.
\newblock Near-optimal sensor placements in gaussian processes: Theory,
  efficient algorithms and empirical studies.
\newblock \emph{Journal of Machine Learning Research}, 9\penalty0
  (Feb):\penalty0 235--284, 2008{\natexlab{b}}.

\bibitem[Kumar et~al.(2015)Kumar, Moseley, Vassilvitskii, and
  Vattani]{kumar2015fast}
Kumar, R., Moseley, B., Vassilvitskii, S., and Vattani, A.
\newblock Fast greedy algorithms in mapreduce and streaming.
\newblock \emph{ACM Transactions on Parallel Computing (TOPC)}, 2\penalty0
  (3):\penalty0 1--22, 2015.

\bibitem[Leskovec et~al.(2007)Leskovec, Kleinberg, and Faloutsos]{LKF07}
Leskovec, J., Kleinberg, J., and Faloutsos, C.
\newblock Graph evolution: Densification and shrinking diameters.
\newblock \emph{ACM transactions on Knowledge Discovery from Data (TKDD)},
  1\penalty0 (1):\penalty0 2--es, 2007.

\bibitem[Lin \& Bilmes(2011)Lin and Bilmes]{lin2011class}
Lin, H. and Bilmes, J.
\newblock A class of submodular functions for document summarization.
\newblock In \emph{Human Language Technologies}, 2011.

\bibitem[Lindgren et~al.(2016)Lindgren, Wu, and Dimakis]{LWD16}
Lindgren, E., Wu, S., and Dimakis, A.~G.
\newblock Leveraging sparsity for efficient submodular data summarization.
\newblock \emph{Advances in Neural Information Processing Systems},
  29:\penalty0 3414--3422, 2016.

\bibitem[Liu \& Vondr{\'a}k(2018)Liu and Vondr{\'a}k]{liu2018submodular}
Liu, P. and Vondr{\'a}k, J.
\newblock Submodular optimization in the mapreduce model.
\newblock \emph{arXiv preprint arXiv:1810.01489}, 2018.

\bibitem[Lojasiewicz(1963)]{L63}
Lojasiewicz, S.
\newblock Une propri{\'e}t{\'e} topologique des sous-ensembles analytiques
  r{\'e}els, in ?les {\'e}quations aux d{\'e}riv{\'e}es partielles (paris,
  1962)? {\'e}ditions du centre national de la recherche scientifique, 1963.

\bibitem[Minoux(1978)]{minoux1978accelerated}
Minoux, M.
\newblock Accelerated greedy algorithms for maximizing submodular set
  functions.
\newblock In \emph{Optimization techniques}, pp.\  234--243. Springer, 1978.

\bibitem[Mirrokni \& Zadimoghaddam(2015)Mirrokni and
  Zadimoghaddam]{mirrokni2015randomized}
Mirrokni, V. and Zadimoghaddam, M.
\newblock Randomized composable core-sets for distributed submodular
  maximization.
\newblock In \emph{Proceedings of the forty-seventh annual ACM symposium on
  Theory of computing}, pp.\  153--162, 2015.

\bibitem[Mirzasoleiman et~al.(2013)Mirzasoleiman, Karbasi, Sarkar, and
  Krause]{mirzasoleiman2013distributed}
Mirzasoleiman, B., Karbasi, A., Sarkar, R., and Krause, A.
\newblock Distributed submodular maximization: Identifying representative
  elements in massive data.
\newblock In \emph{NIPS}, pp.\  2049--2057, 2013.

\bibitem[Mirzasoleiman et~al.(2015)Mirzasoleiman, Badanidiyuru, Karbasi,
  Vondr{\'a}k, and Krause]{MBKV15}
Mirzasoleiman, B., Badanidiyuru, A., Karbasi, A., Vondr{\'a}k, J., and Krause,
  A.
\newblock Lazier than lazy greedy.
\newblock In \emph{Proceedings of the AAAI Conference on Artificial
  Intelligence}, volume~29, 2015.

\bibitem[Mirzasoleiman et~al.(2016{\natexlab{a}})Mirzasoleiman, Badanidiyuru,
  and Karbasi]{MBK16}
Mirzasoleiman, B., Badanidiyuru, A., and Karbasi, A.
\newblock Fast constrained submodular maximization: Personalized data
  summarization.
\newblock In \emph{ICML}, pp.\  1358--1367, 2016{\natexlab{a}}.

\bibitem[Mirzasoleiman et~al.(2016{\natexlab{b}})Mirzasoleiman, Badanidiyuru,
  and Karbasi]{mirzasoleiman2016fast}
Mirzasoleiman, B., Badanidiyuru, A., and Karbasi, A.
\newblock Fast constrained submodular maximization: Personalized data
  summarization.
\newblock In \emph{ICML}, pp.\  1358--1367, 2016{\natexlab{b}}.

\bibitem[Mirzasoleiman et~al.(2017)Mirzasoleiman, Karbasi, and
  Krause]{mirzasoleiman2017deletion}
Mirzasoleiman, B., Karbasi, A., and Krause, A.
\newblock Deletion-robust submodular maximization: Data summarization with
  “the right to be forgotten”.
\newblock In \emph{International Conference on Machine Learning}, pp.\
  2449--2458. PMLR, 2017.

\bibitem[Mitrovic et~al.(2017)Mitrovic, Bogunovic, Norouzi-Fard, Tarnawski, and
  Cevher]{MBNTC17}
Mitrovic, S., Bogunovic, I., Norouzi-Fard, A., Tarnawski, J.~M., and Cevher, V.
\newblock Streaming robust submodular maximization: A partitioned thresholding
  approach.
\newblock In \emph{Advances in Neural Information Processing Systems}, pp.\
  4557--4566, 2017.

\bibitem[Nemhauser \& Wolsey(1978)Nemhauser and Wolsey]{nemhauser1978best}
Nemhauser, G.~L. and Wolsey, L.~A.
\newblock Best algorithms for approximating the maximum of a submodular set
  function.
\newblock \emph{Mathematics of operations research}, 3\penalty0 (3):\penalty0
  177--188, 1978.

\bibitem[Nemhauser et~al.(1978)Nemhauser, Wolsey, and Fisher]{NWF78}
Nemhauser, G.~L., Wolsey, L.~A., and Fisher, M.~L.
\newblock An analysis of approximations for maximizing submodular set functions
  i.
\newblock \emph{Mathematical programming}, 14\penalty0 (1):\penalty0 265--294,
  1978.

\bibitem[Ohsaka \& Yoshida(2015)Ohsaka and Yoshida]{OY15}
Ohsaka, N. and Yoshida, Y.
\newblock Monotone k-submodular function maximization with size constraints.
\newblock \emph{Advances in Neural Information Processing Systems},
  28:\penalty0 694--702, 2015.

\bibitem[Pokutta et~al.(2020)Pokutta, Singh, and Torrico]{PST19}
Pokutta, S., Singh, M., and Torrico, A.
\newblock On the unreasonable effectiveness of the greedy algorithm: Greedy
  adapts to sharpness.
\newblock In \emph{International Conference on Machine Learning}, pp.\
  7772--7782. PMLR, 2020.

\bibitem[Qian \& Singer(2019)Qian and Singer]{qian2019fast}
Qian, S. and Singer, Y.
\newblock Fast parallel algorithms for statistical subset selection problems.
\newblock \emph{Advances in Neural Information Processing Systems},
  32:\penalty0 5072--5081, 2019.

\bibitem[Traud et~al.(2012)Traud, Mucha, and Porter]{TMP12}
Traud, A.~L., Mucha, P.~J., and Porter, M.~A.
\newblock Social structure of facebook networks.
\newblock \emph{Physica A: Statistical Mechanics and its Applications},
  391\penalty0 (16):\penalty0 4165--4180, 2012.

\bibitem[Vondr{\'a}k(2007)]{vondrak2007submodularity}
Vondr{\'a}k, J.
\newblock Submodularity in combinatorial optimization.
\newblock 2007.

\bibitem[Vondr{\'a}k(2008)]{vondrak2008optimal}
Vondr{\'a}k, J.
\newblock Optimal approximation for the submodular welfare problem in the value
  oracle model.
\newblock In \emph{Proceedings of the fortieth annual ACM symposium on Theory
  of computing}, pp.\  67--74, 2008.

\bibitem[Wolsey(1982)]{wolsey1982analysis}
Wolsey, L.~A.
\newblock An analysis of the greedy algorithm for the submodular set covering
  problem.
\newblock \emph{Combinatorica}, 2\penalty0 (4):\penalty0 385--393, 1982.

\bibitem[Yang \& Leskovec(2015)Yang and Leskovec]{YL15}
Yang, J. and Leskovec, J.
\newblock Defining and evaluating network communities based on ground-truth.
\newblock \emph{Knowledge and Information Systems}, 42\penalty0 (1):\penalty0
  181--213, 2015.

\end{thebibliography}
\bibliographystyle{icml2021}

\newpage

\appendix
\section{Missing Analysis for Submodular Functions}
\subsection{Implementation of Method \ref{alg:sm}}\label{app:pseudo}
In general, we can find value $v_j$ at iteration $j$ by iterating through  elements indexed by $i \in \{i_{j-1}+1, i_{j-1}+2, \ldots \}$ until $i^\star$, where $i^{\star}$ is the minimum index such that $f(A_{i^{\star}}) - \sum_{\ell = 1}^{j-1} v_\ell \geq f(a_{i^{\star}})$. In the case where $f(A_{i^{\star} -1}) - \sum_{\ell = 1}^{j-1} v_\ell < f(a_{i^{\star}})$. In this case, $v_j = f(A_{i^{\star} -1}) - \sum_{\ell = 1}^{j-1} v_\ell$. This requires one pass through the elements in $N$. Pseudocode to find $v_j$ is below.
\begin{algorithm}[H]
	\caption{Method to find $v_j$}
	\label{alg:pseudo}
	\begin{algorithmic}
		\INPUT   function $f$,  cardinality constraint $k$
		\STATE For $j = 1$ to $k$:
		\STATE \qquad$i^{\star} \leftarrow \min\{i:  f(A_i) - \sum_{\ell=1}^{j-1} v_\ell \geq f(a_i)\}$
		\STATE \qquad if $f(A_{i^{\star}-1}) - \sum_{\ell=1}^{j-1} v_\ell \geq f(a_{i^\star})$:
		\STATE \qquad\qquad $i^\star \leftarrow i^\star - 1$
		\STATE \qquad\qquad $v_j \leftarrow f(A_{i^{\star}}) - \sum_{\ell=1}^{j-1} v_\ell$
		\STATE \qquad else:
		\STATE \qquad\qquad $v_j \leftarrow f(a_{i^{\star}})$
		\STATE \textbf{return}  $\sum_{j=1}^k v_j$
	\end{algorithmic}
\end{algorithm}
\begin{proposition}
\label{prop:pseudo}
For a monotone submodular function $f$ and cardinality constraint $k$, value $v_j$ in Method \ref{alg:sm} can be computed using Method \ref{alg:pseudo}.
\end{proposition}
\begin{proof}
To maximize value $v_j$ in each iteration $j$, we note that for some $i \in [n]$ either $v_j = f(a_{i})$  or $v_j = f(A_{i}) - \sum_{\ell=1}^{j-1} v_\ell$, so that at least one constraint is tight.

In the case where $v_j = f(a_i)$ for some $i$, we can find $v_j$ by iterating through  elements indexed by $i \in \{i_{j-1}+1, i_{j-1}+2, \ldots \}$ until $i^\star$, where $i^{\star}$ is the minimum index such that $f(A_{i^{\star}}) - \sum_{\ell = 1}^{j-1} v_\ell \geq f(a_{i^{\star}})$ and $v_j = f(a_{i^{\star}})$. Note that this satisfies the constraints of Method 3 and the minimal index corresponds to the maximal possible value of $v_j$.

In the case where $v_j = f(A_{i}) - \sum_{\ell=1}^{j-1} v_\ell$ for some $i$, we note that the value $v_j$ can only be further increased by decrementing $i^\star$ found previously. Thus, the $i$ that gives the maximal value of $f(A_{i}) - \sum_{\ell=1}^{j-1} v_\ell$ for $i \in \{i_{j-1}+1, i_{j-1}+2,\ldots,i^\star -1 \}$ is $i^\star -1$ because is $f(A_i)$ is monotonic in $i$ and $v_j = f(A_{i^\star - 1}) - \sum_{\ell=1}^{j-1} v_\ell$. This satisfies the constraints of Method 3, because $f(a_{i^\star -1}) \geq f(A_{i^\star - 1}) - \sum_{\ell=1}^{j-1} v_\ell$. If this inequality did not hold, then this contradicts the fact that $i^\star$ is the minimal index such that $f(A_{i^{\star}}) - \sum_{\ell = 1}^{j-1} v_\ell \geq f(a_{i^{\star}})$.  We note that this case only occurs when $f(A_{i^\star-1}) - \sum_{\ell = 1}^{j-1} v_\ell \geq f(a_{i^{\star}})$.
\end{proof}

We established in Proposition \ref{prop:pseudo} that Method \ref{alg:sm} passes through the elements in the ground set once. Since this procedure is preceded by a sorting step, this shows that Method \ref{alg:pseudo} has a runtime of $O(n\log n)$.

\subsection{Example of bad instance for Method \ref{alg:sm}}\label{app:bad}
We construct a case where there are two types of elements in the ground set: elements $\mathcal B = \{b_i\}_{i=1}^n$ that have high singleton value and large overlap with other elements in $\mathcal B$ and elements $\mathcal G = \{g_i\}_{i=1}^n$ that have much lower singleton value and smaller overlap with elements in $\mathcal B$. We show that elements in $\mathcal B$ cause Method \ref{alg:sm} to give a poor upper bound of \texttt{OPT} and that the \textsc{Greedy} approximation can be arbitrarily bad. 

For all $i \in [n]$, let $f(b_i) = c$. For any $B \subseteq \mathcal B$, where $b_j \notin B$, $f_{B} (b_j) = 1$, i.e. $b_j$ has high overlap with others elements in $\mathcal B$ and thus, low marginal contribution to $B$. For all $i \in [n]$ and $B \subseteq \mathcal B$, let $f(g_i) = c/2$ and $f_{B}(g_i) =2$, i.e. elements in $\mathcal G$ have lower value than elements in $\mathcal B$, but have high marginal contribution.

We can see that \textsc{Greedy} will first select element $b$ from $\mathcal B$ and the remaining elements from $\mathcal G$ to achieve the solution value of $c + 2\cdot (k-1)$. However, in the case where $n$ is large and there are infinitely many elements in $\mathcal B$, Method \ref{alg:sm} will return a value to upper bound \texttt{OPT} that is $\ubOPT = k\cdot c$. Thus, the \textsc{Greedy} approximation given by Method 3 is $\frac{c + 2\cdot (k-1)}{k\cdot c}$. For large values of $k$ and $c$, this approximation becomes arbitrarily poor.
\section{Additional Details on Experimental Setup}
\subsection{Submodular maximization algorithms}\label{app:alg}
We provide additional details on submodular maximization algorithms and their implementation below.
\begin{itemize}
	\item{{\bf \textsc{Greedy}.}} The greedy algorithm, introduced  by \citet{NWF78}, obtains the optimal $1-1/e$ approximation and is widely considered as the standard algorithm for monotone submodular maximization under a cardinality constraint.  To find a solution set of size $k$, the algorithm adds the element with the largest marginal contribution to the solution set at each iteration.
	\item{{\bf \textsc{Local search}.}} Local search obtains a stronger approximation guarantee of $1/2$ for the more general family of matroid constraints. We use the deterministic algorithm from \citet{NWF78} where the algorithm searches for a pair of elements ($a_1, a_2$), $a_1\in S$ and $a_2 \notin S$, that leads to an improved solution when swapped. In our implementation of the algorithm, we begin with a set $S$ of size $k$ that consists of the top $k$ largest singletons.
	\item{{\bf \textsc{Lazier-than-lazy greedy}.}} Lazier-than-lazy greedy, also called sample greedy, improves the running time of greedy by sampling a small set of size $\frac{n\cdot \ln(1/\epsilon)}{k}$ from the remaining elements, where $\epsilon > 0$. This algorithm has a $1-1/e - \epsilon$ approximation \cite{MBKV15}. For our experiments, we run this algorithm 5 times and average the results.
	\item{{\bf \textsc{Random greedy}.}} Random greedy obtains approximation guarantees for submodular functions that are not necessarily monotonic by introducing randomization into the element selection step. It achieves a $1/e$ approximation for non-monotone functions and a $1-1/e$ approximation for monotone functions \cite{BFN14}. For our experiments, we run this algorithm 5 times and average the results.
\end{itemize}

\subsection{Experimental settings}
We provide additional details on preprocessing and sampling of datasets and objectives below.

\subsubsection{Large instances}\label{app:large}
\begin{itemize}
\item {\bf Influence maximization}: As in \cite{MBK16, BBS18}, we use social network data from the 5,000 largest communities of the {\bf Youtube social network}, which are comprised of 39,841 nodes and 224,234 undirected edges \cite{YL15}. We randomly sample 50 communities to select $n=1,000$ nodes and we select $k$ people that are connected to the largest number of people by maximizing coverage $f(S) = |N_G(S)|$.
\item {{\bf Car dispatch}: In the Uber dispatch application, the goal is to select the $k$ best locations to deploy drivers that cover the maximum number of pickups. As in \cite{BS18, KZK18}, we analyze a dataset of 1,000 locations of {\bf Uber pickups} in Manhattan, New York in April 2014 \cite{FVE}. We assign a weight $w_i$ to each neighborhood $n_i \in N$ that is proportional to the number of trips in the neighborhood, where $N$ is the collection of all neighborhoods. The weighted coverage is defined to be equal to the sum of the weights of neighborhoods $n_i$ that are reachable from at least one pickup location in $S$, i.e., 
	$f(S) =  \sum_{n_i \in N} \mathbbm{1}_{\exists n_j \in S: d(n_i,n_j) \leq R} \cdot w_i,$ where $R = 1.5km$.}
\item {{\bf Movie recommendation}:  We consider a variant of \cite{MBK16, MBNTC17, BBS18, BBS20} and use the {\bf MovieLens 1m dataset} \cite{HK16}, which consists of 6,040 users and 3,706 movies and a total of roughly 1 million ratings, to recommend movies that have both good overall ratings and are highly rated by the most users. Each user $j \in U$ ranks at least one movie $i \in S$ with an integer value from $\{0,...,5\}$ where incomplete rankings are filled in using the standard low-rank matrix completion. We use the completed movie ratings matrix $[r_{ij}]$ of rankings from user $i$ and each movie $j$ to select the $k$ highest ranked movies among users by maximizing $$f(S) = \sum_{i \in S} \sum_{j \in U} r_{ij} + C(S),$$ where $C(S) = \sum_{j \in U} \mathbbm{1}_{\exists i \in S: r_{ij} > 4.5}$. The first additive term represents the total ratings from users $j$ on movie $i$ in set $S$ and the second coverage term is the number of users who ranked any movie in set $S$ highly, i.e. above 4.5.}

\item{{\bf Facility location}: As in \cite{LWD16, PST19}, we sample the {\bf MovieLens dataset} to select a random sample of 500 movies and consider the movie ratings for all 6,040 users.}
\item {\bf Revenue maximization}: We use the objective from \cite{BBS20} to maximize revenue. We use the {\bf CalTech Facebook Network dataset} \cite{TMP12} of 769 Facebook users $N$ and 17,000 edges, and uniformly sample weights $w_{ij} \sim \mathcal U(0,1)$ to denote the revenue value. We select a subset $S$ of users to maximize revenue using $f(S) = \sum_{i \in N} (\sum_{j \in  S} w_{ij})^\alpha$, where $\alpha = 0.9$.
\item{{\bf Feature selection}: In this setting, we use the {\bf Adult Income dataset} from UCI Repository \cite{BM98} and wish to select a subset of relevant features for income prediction in a computationally feasible way. This dataset contains information about 32,561 individuals and we would like to perform classification to predict $Y$, the income of each individual, which is 1 if it is above 50k a year and 0 otherwise. We extract 109 binary features from the data and use a joint entropy objective to select relevant features: $f(S) = H(X_S, Y) = -\sum_{x\in X_S} \sum_{y\in Y} p(x,y)\log p(x,y),$ where $H$ is the entropy function,  $X_S$ is the feature matrix indexed by $S$ and $p(x,y)$ is the joint probability of the occurrences of $x$ and $y$.}
\item {\bf Sensor placement}: As in \cite{KSG08, KMGG08, OY15}, we wish to select sensors to place in different locations around a lab. Instead of the mutual information objective which is not monotonic, we instead consider the entropy objective. We use the {\bf Berkeley Intel Lab dataset} which comprises of 54 sensors that collect temperature information in various locations around the lab. We select $k$ sensors that, in aggregate, provide accurate data readings by maximizing entropy, $f(S) = H(X_S) = - \sum_{x \in X_S} P(x) \log P(x),$ where $H$ is the entropy function and $x$ is the vector of temperature readings of a sensor.
\end{itemize}
\subsubsection{Small instances}\label{app:small}
For these settings, we follow the experimental setup from \citet*{PST19}. For each $k$, we sample a ground set of size $n=2k$ and run the methods for comparison. We average the results for each $k$ across 5 runs where each run corresponds to a different sampled dataset.
\begin{itemize}
	\item {{\bf Facility location}: We use the facility location objective as defined in Section 4.1 on the {\bf MovieLens dataset} \cite{HK16}. }
	\item {{\bf Movie recommendation}: We use the {\bf MovieLens dataset} \cite{HK16} to recommend movies with the movie ratings matrix $[r_{ij}]$ of rankings from user $i$ and each movie $j$. Exactly as \cite{PST19}, we let $f(S) = (\frac{1}{m} \sum_{i=1}^m \sum_{j \in S} r_{ij})^\alpha$, where $m$ is the number of users, and set $\alpha = 0.8$. }
\end{itemize}
\section{Additional Details on Benchmarks}\label{app:bench}
\subsection{Proof of \textsc{Marginal} upper bound}\label{app:bench1}
We provide the missing derivation of the upper bound of $\texttt{OPT}_k$ for the \textsc{Marginal} benchmark below.

\begin{proposition} Let $S_i$ be the \textsc{Greedy} solution of size $i$ and $O$ the optimal set for a monotone submodular function $f$. Then for all $i < j$, 
$$f_{S_i}(S_j) \geq (1-(1-1/k)^{j-i})f_{S_i}(O) .$$

Furthermore, for all $i<j$, this directly gives an upper bound of $\texttt{OPT}_k$.
$$\texttt{OPT}_k \leq  \frac{f(S_j) - (1-1/k)^{j-i}f(S_i)}{1-(1-1/k)^{j-i}}.$$
\end{proposition}
\begin{proof}
We begin by showing $f_{S_i}(S_j) \geq (1-(1-1/k)^{j-i})f_{S_i}(O)$ by induction on $j-i$. 

In the base case, where $j-i=1$, we have $$f_{S_i}(S_j) = f_{S_i}(a_{i+1}) \geq \frac{1}{k} f_{S_i}(O) = (1-(1/k)^1)f_{S_i}(O),$$
as needed.

For the inductive step, we assume that for $j-i = m$, $f_{S_i}(S_j) \geq (1-(1-1/k)^{m})f_{S_i}(O)$ is true.
\begin{align*}
		f_{S_i}(S_{j+1})  & = f_{S_i}(S_j) +  f_{S_j}(a_{j+1}) \\
		& \geq  f_{S_i}(S_j) +  \frac{1}{k}f_{S_j}(O) \\
		& \geq  f_{S_i}(S_j) +  \frac{1}{k}(f(O) - f(S_j)) \\
		& =  f_{S_i}(S_j) +  \frac{1}{k}(f_{S_i}(O) - f_{S_i}(S_j)) \\
		& =  (1-1/k)f_{S_i}(S_j) +  \frac{1}{k}f_{S_i}(O)\\
		&\geq (1-1/k)(1-(1-1/k)^{m})f_{S_i}(O) +  \frac{1}{k}f_{S_i}(O) \\
		& = (1-(1-k)^{m+1})f_{S_i}(O),
\end{align*}
which completes the proof for the statement $f_{S_i}(S_j) \geq (1-(1-1/k)^{j-i})f_{S_i}(O)$.

Finally, we can get the upper bound on $\texttt{OPT}_k$ by rearranging the statement above. By monotonicity, we have $f_{S_i}(S_{j}) \geq (1-(1-k)^{j-i})(\texttt{OPT}_k - f(S_i)$. Rearranging, gives
\begin{align*}
		\texttt{OPT}_k & \leq  \frac{f(S_{j}) - f(S_i) + (1-(1-k)^{j-i}) f(S_i)}{1-(1-k)^{j-i}}\\
		& =\frac{f(S_{j}) -(1-k)^{j-i} f(S_i)}{1-(1-k)^{j-i}}.
\end{align*}
\end{proof}
\subsection{Details on \textsc{Sharpness} benchmark} \label{app:bench_sharp}
In our experiments, we consider the property of Dynamic Submodular Sharpness which yields the strongest sharpness approximation from \citet*{PST19}. For completeness, we state the definition and approximation below. For more details, see Section 1.2.2 of  \citet*{PST19}.

\begin{definition}
A non-negative monotone submodular function $f: 2^N \to \mathbb R$ is said to be dynamic $(c, \theta)$-submodular sharp, where $c = (c_0, c_1, ...,c_{k-1}) \in [1, \infty)^k$ and $\theta = (\theta_0, \theta_1, ...,\theta_{k-1}) \in [0, 1]^k$, if there exists an optimal solution $O$ such that for any subset $S\subseteq N$ with $|S|\leq k$, the function satisfies 
$$\max_{a \in O\backslash S} f_S(a) \geq \frac{1}{kc_{|S|}} \big(f(O) - f(S) \big)^{1-\theta_{|S|}} f(O)^{\theta_{|S|}} .$$
\end{definition}
This inequality can be interpreted as the submodular version of the Polyak-Lojasiewicz inequality \cite{L63} with the $\ell_\infty$-norm. Additionally, any monotone submodular function $f$ is $(c,\theta)$-dynamic submodular sharp as $c_i \to 1$ and  $\theta_i \to 0$. Now, we introduce the approximation in the following theorem.
\begin{customthm}{4}
Consider a non-negative monotone submodular function $f: 2^N \to \mathbb R$ that is dynamic $(c, \theta)$-submodular sharp, with parameters $c = (c_0, c_1, ...,c_{k-1}) \in [1, \infty)^k$ and $\theta = (\theta_0, \theta_1, ...,\theta_{k-1}) \in [0, 1]^k$. Then, \textsc{Greedy} returns a set $S_g$ such that 
$$f(S_g) \geq  \Big[ 1 - \bigg(\Big(\big(1-\frac{\theta_0}{c_0k} \big)^{\frac{\theta_1}{\theta_0}} -\frac{\theta_1}{c_1k}\Big)^{\frac{\theta_2}{\theta_1}} - \cdots -  \frac{\theta_{k-1}}{c_{k-1}k}\bigg)^{\frac{1}{\theta_{k-1}}}  \Big] \cdot f(O).$$
\end{customthm}
To compute the parameters $c = (c_0, c_1, ...,c_{k-1}) \in [1, \infty)^k$ and $\theta = (\theta_0, \theta_1, ...,\theta_{k-1}) \in [0, 1]^k$, we follow a search detailed in \citet*{PST19}. We sequentially iterate over possible values of $c$ in a fixed range $[1, 3]$ with a discretization of 0.01. Given $c_i$, we compute $\theta_i = \min_{|S| \leq k} \Big( \frac{\log (kc_iW_2(S)/ W(S)}{\log(\texttt{OPT}/W(S))}\Big)$, where
$$W(S) = \texttt{OPT} - f(S) \qquad \text{and} \qquad W_2(S) = \max_{a\in O\backslash S} f_S(a).$$

Once $c_i$ and $\theta_i$ are computed, we compute the approximation factor. If the approximation factor improves, we continue the search and update $c_i$; otherwise, we stop. We do this for all $i$. 

\subsection{Details of \textsc{Curvature} upper bound}\label{app:bench_curv}
For large-instances, we are unable to compute \textsc{Curvature} and instead, use find a lower bound on the curvature parameter so that we can derive an upper bound on the curvature approximation. Since computing the curvature parameter $c$ requires a search across all sets, we instead upper bound the approximation using the following methodology. We find an element $a^\star$ that maximizes curvature parameter: $c= \max_{a\in N} 1-\frac{f_{N\backslash a}(a)}{f(a)}$. We then use $a^\star$ to greedily compute an upper bound of the approximation by adding $k$ elements to $S$ to maximize the following function $c(S)= 1-\frac{f_S(a^\star)}{f(a^\star)}$. The resulting curvature parameter can be used to upper bound the true curvature approximation, $1-e^{-c}/c$, for the \textsc{Greedy} algorithm. This procedure was used in the experiments that are shown in Figure 2.  

\section{Additional Experimental Results}
For completeness, we include complete experimental results for the running time of methods on small instances for each $k \in [1,...,10]$ and $n =2k$. In the paper, we show a snapshot of the results in Table 1 for the MovieLens facility location setting. Here, we show the full results for each $k$ in Table 2 and Table 3.
Both of the objectives show the same trend, where the methods that require brute-force, \textsc{OPT}, \textsc{Curvature} and \textsc{Sharpness}, all increase exponentially in running time with respect to $k$. In contrast, \textsc{Dual} is magnitudes faster than these methods and is scalable. The running time was computing by averaging 5 repeated executions.
\begin{table}[!htb]
    \begin{minipage}{.48\linewidth}
      \centering
\resizebox{0.98\columnwidth}{!}{%
       \begin{tabular}{l|lllll}
$k$  & \textsc{Marg.} & \textsc{Dual}   & \textsc{Opt} & \textsc{Curv.} & \textsc{Sharp.} \\ \hline
1  & 3.84 e-5 & 2.34 e-4 & 1.84 e-4  & 1.97 e-4     & 3.07 e-4      \\
2  & 2.69 e-5 & 1.06 e-3 & 3.07 e-4  & 5.15 e-4     & 1.49 e-3      \\
3  & 4.17 e-5 & 2.53 e-3 & 6.15 e-4  & 3.05 e-3     & 6.74 e-3      \\
4  & 5.30 e-5 & 3.58 e-3 & 1.81 e-3  & 0.0108     & 0.0263      \\
5  & 7.24 e-5 & 6.37 e-3 & 6.17 e-3  & 0.0546     & 0.108      \\
6  & 9.60 e-5 & 7.55 e-3 & 0.0212  & 0.176     & 0.448      \\
7  & 1.25 e-4 & 9.66 e-3 & 0.0734  & 0.819     & 1.92      \\
8  & 1.59 e-4 & 0.0129 & 0.249   & 3.55      & 9.00      \\
9  & 1.98 e-4 & 0.0178 & 0.861  & 15.6     & 43.0      \\
10 & 3.17 e-4 & 0.0246  & 3.43    & 74.7       & 187      
\end{tabular}
}
\caption{Average runtimes (seconds) of benchmark methods on MovieLens facility location setting.}
    \end{minipage}%
    \hfill
    \begin{minipage}{.48\linewidth}
      \centering
      \resizebox{0.98\columnwidth}{!}{%
        \begin{tabular}{l|lllll}
$k$  & \textsc{Marg.} & \textsc{Dual}   & \textsc{Opt} & \textsc{Curv.} & \textsc{Sharp.} \\ \hline
1  & 1.59 e-5 & 3.19 e-4 & 1.90 e-4  & 1.71 e-4     & 9.74 e-4      \\
2  & 2.66 e-5 & 2.07 e-3 & 3.76 e-4  & 1.05 e-3     & 2.12 e-3      \\
3  & 3.68 e-5 & 4.31 e-3 & 8.58 e-4  & 4.26 e-3     & 1.00 e-2      \\
4  & 5.24 e-5 & 7.81 e-3 & 3.35 e-3  & 0.0183     & 0.0483      \\
5  & 7.34 e-5 & 0.0127 & 9.91 e-3  & 0.0827     & 0.209      \\
6  & 1.03 e-4 & 0.0177 & 0.0348  & 0.360     & 0.807      \\
7  & 1.28 e-4 & 0.0249 & 0.123  & 1.41     & 3.29      \\
8  & 1.90 e-4 & 0.0319 & 0.450   & 7.18      & 15.4      \\
9  & 1.76 e-4 & 0.0342 & 1.61  & 29.8     & 65.3      \\
10 & 2.20 e-4 & 0.0455  & 5.56    & 117       & 277      
\end{tabular}
}
\caption{Average runtimes (seconds) of benchmark methods on MovieLens movie recommendation setting.}
    \end{minipage} 
\end{table}

\end{document}